\documentclass[11pt]{article}

\usepackage{fullpage}

\usepackage[T1]{fontenc}
\usepackage[utf8]{inputenc}
\usepackage{lmodern}
\usepackage{microtype}

\usepackage{amsmath,mathtools}
\usepackage{amsfonts,amssymb}
\usepackage{amsthm,thmtools}
\usepackage{bm}
\usepackage{bbm}
\usepackage{accents}

\usepackage{graphicx}
\usepackage{xcolor}
\usepackage{booktabs}
\usepackage{array}
\usepackage{tabularx}
\usepackage{threeparttable}
\usepackage{caption}
\usepackage{subcaption}
\usepackage{makecell}
\usepackage{pifont}

\usepackage{enumitem}

\usepackage{algorithm}
\usepackage{algpseudocode}

\usepackage[round]{natbib}
\usepackage[colorlinks=true,citecolor=blue,linkcolor=red,urlcolor=magenta]{hyperref}
\usepackage{cleveref}

\crefformat{equation}{(#2#1#3)}
\Crefformat{equation}{(#2#1#3)}
\crefrangeformat{equation}{(#3#1#4)--(#5#2#6)}
\Crefrangeformat{equation}{(#3#1#4)--(#5#2#6)}
\crefmultiformat{equation}{(#2#1#3)}{ and (#2#1#3)}{, (#2#1#3)}{, and (#2#1#3)}
\Crefmultiformat{equation}{(#2#1#3)}{ and (#2#1#3)}{, (#2#1#3)}{, and (#2#1#3)}

\input{macro}

\title{Minimax Optimal Regret for Causal Logistic Bandits with Counterfactual Fairness}

\author{
Junhyuk Huh$^{1}$ \quad Seoungbin Bae$^{2}$ \quad Dabeen Lee$^{3}$\\[0.3em]
$^{1}$Faculty of Mathematics, University of Cambridge\\
$^{2}$Department of Industrial \& Systems Engineering, KAIST\\
$^{3}$Department of Mathematical Sciences, Seoul National University\\
\texttt{jh2556@cam.ac.uk, sbbae31@kaist.ac.kr, dabeenl@snu.ac.kr}
}

\date{}

\begin{document}
\maketitle

\begin{abstract}
We study causal logistic bandits with counterfactual fairness constraints. The causal structure is given through known factual and counterfactual feature maps that share an unknown logistic reward parameter, but the learner observes only factual rewards. Consequently, the directions determining counterfactual feasibility need not be identifiable from the available feedback. The closest prior analyses either omit a coverage condition or impose a comparatively strong one, and do not establish matching lower bounds. We first show that some coverage condition is necessary: without a coverage-type restriction, factually indistinguishable environments with different optimal fair actions force $\Omega(T)$ expected joint loss. Under a weaker full-rank condition on the factual covariance pooled across actions, we identify a target-specific information scale $V_\star$ that measures the difficulty of estimating rewards and counterfactual effects from factual feedback. We construct worst-case families satisfying this condition on which every policy incurs expected joint loss $\Omega\left(\left[V_\star\min\{\log K,d\}\right]^{1/3}T^{2/3}\right)$. We also give an explore--then--exploit procedure tuned using $V_\star$ and an adaptive algorithm that does not require its value. Both algorithms achieve $\max\{R_T,V_T\}=\widetilde{O}\left(\left[V_\star\min\{\log K,d\}\right]^{1/3}T^{2/3}+\kappa d/\sigma_0^2\right)$, where $R_T$ is regret relative to the best fair action and $V_T$ denotes the cumulative stage-wise positive violations. Thus the upper and lower bounds match in their leading dependence on $T$, $V_\star$, and $\min\{\log K,d\}$, up to logarithmic factors.
\end{abstract}

\section{Introduction}

Contextual bandits model sequential decisions in which a learner observes a context, chooses an action, and receives feedback only for the chosen action. Causal bandits use known causal structure to share information across interventions \citep{lattimore2016causal,sen2017identifying,lu2020regret}. We study a \textit{causal logistic} bandit in which the causal structure is given through known factual and counterfactual feature maps. These maps encode a specified intervention on a protected attribute, while the logistic reward parameter shared by the two maps is unknown. Our focus is the resulting constrained bandit learning problem: the learner seeks high rewards while keeping the factual--counterfactual mean reward difference within a prescribed tolerance.

The difficulty is that the learner observes only the factual reward generated by the chosen action. The counterfactual reward after intervention, and hence the statistic used to determine fairness, is never observed \citep{chen2025causal}. The learner must use factual feedback both to identify rewarding actions and to determine which actions are fair. We ask whether factual observations provide enough information about the parameter directions that determine the counterfactual effect.

The closest predecessors of our work are \citet{chen2025causal} and \citet{chen2026causal}. The former introduces causal logistic bandits with a counterfactual fairness constraint. Its assumptions, however, do not ensure that factual observations cover the parameter directions needed to estimate the counterfactual feedback. This turns out to be a fundamental problem: if counterfactual feedback depends on a parameter component that factual feedback cannot identify, then no policy can guarantee both sublinear regret and sublinear violation. The latter makes this issue explicit by imposing a covariate diversity assumption and adding forced exploration, obtaining sublinear guarantees. Its assumption, however, requires every action separately to cover every parameter direction, whereas we only require the covariance pooled across actions to be full rank. Moreover, \citet{chen2026causal} does not provide a corresponding lower bound establishing whether its upper rate is unavoidable.

Under the pooled coverage condition, the learning difficulty depends on how accurately factual observations can estimate rewards and counterfactual effects. We measure this through a target scale $V_\star$, the largest local variance for estimating these quantities from uniform factual exploration.

\paragraph{Results.} We write $R_T$ for regret relative to the highest-reward fair oracle and $V_T$ for the cumulative stage-wise positive violations, which cannot cancel across rounds. Their sum $R_T+V_T$ is the joint loss. These quantities are defined precisely in Section~\ref{sec:model}. Our contributions are threefold.
\begin{enumerate}
\item \emph{Necessity of a coverage condition.} Without a coverage-type restriction, we construct two environments that generate identical observable history distributions under every policy but have different optimal fair actions. Every policy therefore incurs $\E[R_T+V_T]=\Omega(T)$ in at least one of them, ruling out uniform guarantees of both sublinear regret and sublinear violation over the unrestricted class. For the next two results, we adopt a sufficient condition: the factual covariance averaged over uniform action sampling is full rank (Assumption~\ref{ass:coverage}), while individual action covariances may be singular.

\item \emph{Lower bounds on worst-case families.} Even under this coverage condition, we construct a two-environment, three-action family on which every policy incurs $\Omega(V_\star^{1/3}T^{2/3})$ expected joint loss in at least one environment. We then extend the construction to $M$ environments and $K$ actions, obtaining an additional $(\log M)^{1/3}$ factor, with $\log M=\Theta(\min\{\log K,d\})$.

\item \emph{Matching upper bounds with adaptation to $V_\star$.} Our explore--then--exploit algorithm with a fixed exploration length $m\geq n_0$ guarantees with high probability that
 \[
   \max\{{R_T, V_T}\}
   \le m+1024T\sqrt{\frac{V_\star\ell}{m}}
 \]
where $n_0$ is an initial estimation threshold and $\ell=O(\log(T/\delta)+\min\{\log K,d\})$ is a confidence factor. If $V_\star$ is known, balancing the two terms gives $O\bigl(n_0+(V_\star\ell)^{1/3}T^{2/3}\bigr)$ joint loss. When $V_\star$ is unknown, an adaptive procedure uses observable confidence widths to choose when to collect more exploration data, retaining the same $V_\star$-dependent rate without taking $V_\star$ as an input. These upper bounds match the lower bound on worst-case families in their leading dependence on $T$, $V_\star$, and $\min\{\log K,d\}$, up to logarithms.
\end{enumerate}

Table~\ref{tab:results} summarizes our results, and notation $\widetilde O$ hides logarithmic factors in $T$, $1/\delta$, and fixed constants.
\begin{table}[t]
\centering
\captionsetup{skip=4pt}
\caption{Upper and lower bounds for the joint loss, which is the sum of regret and violation.}
\label{tab:results}
\begin{tabular}{@{}lll@{}}
\toprule
Result & Knowledge of $V_\star$ & Guarantee \\
\midrule
Fixed ETE (Alg.~\ref{alg:fixed-ete}) & Used to choose $m$
& $\widetilde O\bigl([V_\star\min\{\log K, d\}]^{1/3}T^{2/3}
  +\kappa d/\sigma_0^2\bigr)$ \\
Adaptive ETE (Alg.~\ref{alg:adaptive-ete}) & Not assumed
& $\widetilde O\bigl([V_\star\min\{\log K, d\}]^{1/3}T^{2/3}
  +\kappa d/\sigma_0^2\bigr)$ \\
Lower bound (Theorem~\ref{thm:general}) & ---
& $\Omega\bigl((V_\star\min\{\log K, d\})^{1/3}T^{2/3}\bigr)$ \\
 Proposition~\ref{prop:no-coverage} & ---
& $\Omega(T)$ without a coverage assumption\\
\bottomrule
\end{tabular}
\end{table}
The two upper bounds hold with probability at least $1-\delta$ and control $\max\{R_T,V_T\}$. The lower bounds concern expected $R_T+V_T$ on worst-case families. A discussion of related work is provided in Appendix~\ref{sec:additional-related-work}.

\section{Problem Formulation}
\label{sec:model}

Fix a horizon $T\ge2$, a finite action set $\cA=[K]$, a context distribution $P_X$, and a parameter set $\Theta\subset\R^d$.  Let $\mu(u)={1}/({1+e^{-u}})$ and note that  $\dot\mu(u)=\mu(u)(1-\mu(u))$. For every context--action pair $(x,a)$ and $\theta\in\Theta$, define
\begin{gather*}
 r_\theta(x,a)=\mu\bigl(\phi(x,a)^\top\theta\bigr),
  \quad 
 r_\theta^G(x,a)=\mu\bigl(\phi^G(x,a)^\top\theta\bigr),\quad \Delta_\theta(x,a)=r_\theta(x,a)-r_\theta^G(x,a).
\end{gather*}
The feature maps $\phi$ and $\phi^G$, sending a context-action pair to a feature vector in $\R^d$, are known to the learner. Here $\phi^G(x,a)$ is an alternative notation for $\phi(x',a)$, where $x'$ is a context $x$ after counterfactual intervention. We call $\phi(x,a)^\top\theta$ the factual logit and $\phi^G(x,a)^\top\theta$ the counterfactual logit. We use the term \emph{effect} for the factual--counterfactual statistic $\Delta_\theta(x,a)$.

In each round $t \in [T]$, a context $X_t\sim P_X$ is drawn independently. The learner observes $\phi(X_t,a)$ and $\phi^G(X_t,a)$ for every $a\in\cA$ and chooses an action $A_t$. A reward is then drawn as $R_t\sim\Bern\bigl(r_{\theta^\star}(X_t,A_t)\bigr)$. The true parameter $\theta^\star$ is unknown, and neither $r_{\theta^\star}^G(X_t,A_t)$ nor $\Delta_{\theta^\star}(X_t,A_t)$ is observed. The use of the same parameter $\theta^\star$ in the two feature maps is the structural assumption that permits inferring counterfactual information from factual observations.
\begin{assumption}
\label{ass:bounded}
$\Theta$ is convex and compact, $\theta^\star\in\Theta$, and $\norm{\theta}_2\le S$ for every $\theta\in\Theta$.  Moreover, $\norm{\phi(x,a)}_2\le1, \norm{\phi^G(x,a)}_2\le1$ for every $(x,a)$.
\end{assumption}
We define the global inverse curvature $ \kappa =\sup_{\theta\in\Theta,x,a} \max\{ \dot\mu(\phi(x,a)^\top\theta)^{-1}, \dot\mu(\phi^G(x,a)^\top\theta)^{-1} \}$, which is finite by Assumption~\ref{ass:bounded}. These boundedness and curvature conditions are standard in logistic bandits \citep{faury2020improved}. We assume that $\kappa$ is known to the learner; a known upper bound can replace it in the algorithms and their guarantees.

Fix a threshold $\tau\in[0,1]$ and define a true feasible set $\cA_\tau(x) =\{a\in\cA:|\Delta_{\theta^\star}(x,a)|\le\tau\}$.
\begin{assumption}[Fair sets are nonempty]
\label{ass:fair}
$\cA_\tau(x)\ne\varnothing$ for $P_X$-almost every $x$.
\end{assumption}

Using any fixed deterministic tie-breaking, let $a^\star(x) \in\arg\max_{a\in\cA_\tau(x)}r_{\theta^\star}(x,a)$. The cumulative regret relative to this fair oracle and the cumulative stage-wise positive violations are
\begin{gather*}
 R_T=\sum_{t=1}^T
 \left[r_{\theta^\star}(X_t,a^\star(X_t))
       -r_{\theta^\star}(X_t,A_t)\right],\quad V_T=\sum_{t=1}^T
 \pos{|\Delta_{\theta^\star}(X_t,A_t)|-\tau}.
\end{gather*}
We call the sum $R_T+V_T$ the \emph{joint loss}.

\section{Coverage and the Target Scale \texorpdfstring{$V_\star$}{V-star}}
\label{sec:vstar}

\subsection{Coverage and Learnability}
\label{sec:coverage-necessity}

Let $X\sim P_X$ and $A\sim\mathrm{Unif}(\cA)$ independently. Define the \emph{pooled factual covariance} $\Sigma$ and Fisher information $I(\theta)$ as
\begin{gather*}
 \Sigma
 =\E\bigl[\phi(X,A)\phi(X,A)^\top\bigr],\quad I(\theta)
 =\E\left[
 \dot\mu\bigl(\phi(X,A)^\top\theta\bigr)
 \phi(X,A)\phi(X,A)^\top
 \right].
\end{gather*}

\begin{assumption}[Coverage]
\label{ass:coverage}
$\lambda_{\min}(\Sigma)=\sigma_0^2>0$, and $\sigma_0$ is known to the learner.
\end{assumption}

The pooled full-rank condition in Assumption~\ref{ass:coverage} is also used in finite-armed contextual generalized linear bandits \citep{li2017provably,jun2021improved}. We will show that the same assumption on factual features suffices for sublinear regret and violation under counterfactual fairness constraints.

Without a restriction on parameter directions that factual observations cannot identify, no policy can guarantee both sublinear regret and sublinear violation uniformly over the model class.

\begin{proposition}[Linear loss without counterfactual learnability]
\label{prop:no-coverage}
There exists a two-dimensional, two-action causal logistic bandit with $\theta_0,\theta_1$ satisfying Assumptions~\ref{ass:bounded} and~\ref{ass:fair} such that, for every policy $\pi$,
$$\max\left\{
\E_{0,\pi}[R_T], \E_{1,\pi}[V_T] \right\} \ge \frac{T}{6},\quad \max_{j\in\{0,1\}} \E_{j,\pi}[R_T+V_T] \ge \frac{T}{8}$$ where $\E_{j,\pi}$ denotes expectation under $\theta^\star=\theta_j$ and policy $\pi$ for $j\in\{0,1\}$.
\end{proposition}

The construction for Proposition~\ref{prop:no-coverage} uses a single context $x_0$, two actions, and two parameters $\theta_0,\theta_1$. The factual reward laws are identical under both parameters, but one action's counterfactual reward differs, changing the optimal fair action. The full proof is in Appendix~\ref{app:no-coverage-proof}.

The issue in the construction is that factual features are orthogonal to a parameter direction that changes the counterfactual effect. Assumption~\ref{ass:coverage} rules out this possibility by requiring $\Sigma$ to be positive definite, so $\ker(\Sigma)=\{0\}$. Individual action covariances may still be singular. Thus, it is weaker than the covariate-diversity condition in Assumption~4 of \cite{chen2026causal}.

\subsection{The Target Scale \texorpdfstring{$V_\star$}{V-star}}

We now use Assumption~\ref{ass:coverage} to quantify the factual information available for the reward and effect directions. Write $I_\star=I(\theta^\star)$.  By the definition of $\kappa$, we have $I_\star\succeq\frac{1}{\kappa}\Sigma \succeq\frac{\sigma_0^2}{\kappa}I_d$ and thus $I_\star^{-1}\preceq\frac{\kappa}{\sigma_0^2}I_d$. With the true parameter $\theta^\star$, define $q_r(x,a)=\dot\mu\bigl(\phi(x,a)^\top\theta^\star\bigr)\phi(x,a)$, $q_G(x,a)=\dot\mu\bigl(\phi^G(x,a)^\top\theta^\star\bigr)\phi^G(x,a)$, and $q_\Delta(x,a) =q_r(x,a)-q_G(x,a)$. These are the gradients of factual reward, counterfactual reward, and the effect at $\theta^\star$, respectively.

\begin{definition}[Target scale]
\label{def:vstar}
Given $I_\star=I(\theta^\star)$, the target scale is defined as
\begin{equation*}
 V_\star
 =\esssup_{x\sim P_X}\max_{a\in\cA}
 \max\left\{
 \normM{q_r(x,a)}{I_\star^{-1}}^2,
 \normM{q_\Delta(x,a)}{I_\star^{-1}}^2
 \right\}.
\end{equation*}
\end{definition}

The matrix $I_\star$ measures the information about each parameter direction provided by uniform factual exploration. The vectors $q_r$ and $q_\Delta$ record which directions are relevant for reward and fairness. Hence $q^\top I_\star^{-1}q$ gives the local variance for estimating the target in direction $q$ from factual data. The essential supremum bounds this variance for $P_X$-almost every context.

Definition~\ref{def:vstar} also gives $ \normM{q_r}{I_\star^{-1}}\le\sqrt{V_\star}$ and $ \normM{q_\Delta}{I_\star^{-1}}\le\sqrt{V_\star}$. Moreover, since  $q_G=q_r-q_\Delta$, it follows that $\norm[I_\star^{-1}]{q_G}\le2\sqrt{V_\star}$ and $ \norm[I_\star^{-1}]{q_r} +\norm[I_\star^{-1}]{q_G} \le3\sqrt{V_\star}$ hold almost surely. Furthermore, $\dot\mu\le1/4$ and the feature norms are at most one, so $V_\star\le \kappa/(4\sigma_0^2)$. In the lower bound, the inequality $(q^\top h)^2 \le \bigl(q^\top I^{-1}q\bigr)\bigl(h^\top Ih\bigr)$ relates the change in a target along a parameter direction $h$ to the information available in that direction.

\section{Lower Bounds}
\label{sec:lower}

We derive our lower bounds using information-theoretic arguments for stochastic bandits~\citep{lattimore2019bandit}. To the best of our knowledge, these are the first minimax lower bounds for causal logistic bandits with counterfactual fairness constraints. We construct environments with different optimal fair actions that can be distinguished only by choosing low-reward actions. The factual and counterfactual features are chosen so that the information needed to distinguish these environments is determined by $V_\star$, while the pooled factual covariance $\Sigma$ remains full rank.

We compare environments with different true parameters: in environment $j$, the true parameter is $\theta^\star = \theta_j$. Write $V_{\star,j}$ for the target scale in Definition~\ref{def:vstar} evaluated at $\theta_j$. The resulting lower bounds on the worst-case expected joint loss match the upper bounds in Section~\ref{sec:upper} in their leading dependence on $T$ and $V_\star$, up to logarithmic factors.

\subsection{Two-Environment Three-Action Construction}
\label{sec:binary}

The first construction has dimension $d=2$, three actions, and two possible parameters $\theta_0$ and $\theta_1$. Contexts are drawn independently and uniformly from $\{-1,1\}$. The three actions, $H$ (high-reward), $S$ (safe), and $E$ (informative), have the following roles.

\begin{table}[H]
\centering
\captionsetup{skip=4pt}
\caption{Roles of the three actions in the two-dimensional family.}
\label{tab:binary-actions}
\begin{tabular}{@{}llll@{}}
\toprule
Action & Reward role & Information role & Fairness role \\
\midrule
$H$ & high reward & uninformative & fair only in environment $0$ \\
$S$ & slightly lower reward & uninformative & always fair \\
$E$ & low reward & informative  & always fair \\
\bottomrule
\end{tabular}
\end{table}

Set $v_0=\dot\mu(3/4)$ and $\Delta_E=1/2-\mu(-3/4)-1/128>0$. Fix $B\ge1$ and $\rho\in(0,1]$, write $v_B=\dot\mu(B)$, and define $C=(1+\rho^2)/(4e^{1/4}\rho^2v_B)$. The full construction is given in the appendix.

For a parameter $\alpha$ to be chosen later, the family has the following properties: (1) in environment $0$, $H$ is the optimal fair action, and every other action incurs a joint cost of at least a constant multiple of $\alpha$; (2) in environment $1$, $S$ is the optimal fair action, and choosing $H$ incurs a joint cost of at least a constant multiple of $\alpha$; (3) each selection of $E$ incurs a joint cost of at least $\Delta_E$ in both environments; (4) $H$ and $S$ have the same factual reward laws in the two environments; (5) the KL divergence between the two environments for a single observation from $E$ is at most $\alpha^2/(8C)$; (6) both environments have target scale $V_{\star,j}\asymp C$.

\begin{theorem}[Binary lower bound]
\label{thm:binary}
There exist universal constants $c_1,C_0>0$ such that, for every $B\ge1$, every $\rho\in(0,1]$, and every horizon $T\ge C_0C$, there exists a two-environment family with $d=2$ and $K=3$, satisfying Assumptions~\ref{ass:bounded}--\ref{ass:coverage} and $V_{\star,j}\asymp C$, for which
\begin{equation}
 \inf_\pi\max_{j\in\{0,1\}}
 \E_{j,\pi}[R_T+V_T]
 \ge c_1C^{1/3}T^{2/3}.
 \label{eq:binary-lower}
\end{equation}
\end{theorem}

Only observations from the costly action $E$ distinguish the two environments. Any policy must therefore incur worst-case expected joint loss $\Omega(C/\alpha^2)$ from exploration or $\Omega(\alpha T)$ from uncertainty about whether $H$ is fair. Balancing these costs with $\alpha\asymp(C/T)^{1/3}$ yields the lower bound.

\subsection{The General Lower Bound}

We extend Section~\ref{sec:binary} to $M$ environments by replacing the single high-reward action with candidates $H_1,\ldots,H_M$. All candidates have factual reward $1/2$, but separated parameter directions and suitable counterfactual features make $H_j$ the only fair candidate in environment $j$. Every other action incurs joint cost $\Omega(\alpha)$. A safe action remains available, and only costly informative actions distinguish the environments. The action budget and a packing in $d-1$ dimensions allow $\log M\asymp\min\{\log K,d\}$. The exact parameter choices and construction are given in Appendix~\ref{app:general-construction}.

Let $D$ denote the information scale defined in the theorem below, with $V_{\star,j}\asymp D$. Each informative selection incurs constant joint cost, while the KL divergence between any two environments for one such observation is $O(\alpha^2/D)$. Reliably identifying $H_j$ requires information of order $\log M$. Fano's inequality~\citep{cover1991elements} then forces worst-case expected joint loss $\Omega(D\log M/\alpha^2)$ from exploration or $\Omega(\alpha T)$ from failing to select $H_j$. Balancing these costs yields the rate $(D\log M)^{1/3}T^{2/3}$ below.

\begin{theorem}[General lower bound]
\label{thm:general}
There exist constants $c_2,C_2>0$ with the following property. Fix $B\ge1$, $\rho\in(0,1]$, $d\ge46$, and $K\ge8$, and set $D=(d-1)(1+\rho^2)/(\rho^2\dot\mu(B))$. For every horizon $T\ge C_2D\min\{\log K,d\}$, there exists a finite family of $d$-dimensional, $K$-action environments indexed by $j\in\mathcal J$, each satisfying Assumptions~\ref{ass:bounded}--\ref{ass:coverage} and $V_{\star,j}\asymp D$, such that
\begin{equation}
 \inf_\pi\max_{j\in\mathcal J}
 \E_{j,\pi}[R_T+V_T]
 \ge c_2\bigl(D\min\{\log K,d\}\bigr)^{1/3}T^{2/3}.
 \label{eq:general-lower}
\end{equation}
\end{theorem}

\begin{remark}[Connection to partial monitoring]
The $T^{2/3}$ rate has a similar origin to that in hard partial-monitoring problems: distinguishing competing actions requires observations from actions with a constant loss gap~\citep[Appendix~G]{kirschner2020information}. Here, the relevant hidden information is counterfactual feasibility, and learning it requires costly factual observations. Balancing exploration costs against the loss from choosing the wrong action yields the same horizon dependence.
\end{remark}

\section{Algorithms and Upper Bounds}
\label{sec:upper}

We give two algorithms that use uniform exploration to estimate rewards and counterfactual effects, then select actions using confidence intervals. More exploration incurs regret and violation but narrows these intervals. The algorithms balance this cost against the regret and violation during exploitation, which we bound by the interval widths.

\subsection{Confidence Intervals and Optimistic Screening}
\label{sec:confidence}

Both algorithms use the following estimates and action-selection rule. We describe them for $n$ exploration observations; Sections~\ref{sec:fixed-ete} and~\ref{sec:adaptive-ete} specify when to collect these observations and update the estimates.

Let $\mathcal D=((X_s^E,A_s^E,R_s^E))_{s=1}^n$ denote the observations retained for estimation, listed in order of collection, with $\phi_s^E=\phi(X_s^E,A_s^E)$. The algorithms collect these observations by scheduling exploration before seeing its contexts and choosing actions uniformly. Using all $n$ observations, define the unconstrained logistic MLE by
\begin{equation}
 \wh\theta_n
 \in\arg\max_{\theta\in\R^d}
 \sum_{s=1}^n
 \left[
 R_s^E(\phi_s^E)^\top\theta
 -\log\bigl(1+e^{(\phi_s^E)^\top\theta}\bigr)
 \right].
 \label{eq:mle}
\end{equation}
The same estimate gives the factual prediction $\mu(\phi(x,a)^\top\wh\theta_n)$ and the counterfactual prediction $\mu(\phi^G(x,a)^\top\wh\theta_n)$. To quantify their estimation errors, we use the estimated Fisher matrix
\begin{equation}
 \wh H_n
 =\sum_{s=1}^n
 \dot\mu\bigl((\phi_s^E)^\top\wh\theta_n\bigr)
 \phi_s^E(\phi_s^E)^\top.
 \label{eq:empirical-fisher}
\end{equation}
Let $\ell>0$ be a confidence factor, whose value we specify below. When the MLE is finite and $\wh H_n$ is invertible, define the logit radius $b_n(z)=11\sqrt\ell\,\normM{z}{\wh H_n^{-1}}$. The radius $b_n(z)$ is larger when $z$ points in a direction with little information in $\wh H_n$. Applying the increasing function $\mu$ to the two endpoints of the logit interval gives the mean-reward interval $[\mu(z^\top\wh\theta_n-b_n(z)),\mu(z^\top\wh\theta_n+b_n(z))]$. For a pair $(x,a)$, denote this interval by $[\underline r_n(x,a),\overline r_n(x,a)]$ when $z=\phi(x,a)$ and by $[\underline r_n^G(x,a),\overline r_n^G(x,a)]$ when $z=\phi^G(x,a)$. Since the effect is the factual mean minus the counterfactual mean, its interval is $[\underline\Delta_n(x,a),\overline\Delta_n(x,a)] = [\underline r_n(x,a)-\overline r_n^G(x,a), \overline r_n(x,a)-\underline r_n^G(x,a)]$. The largest reward or effect interval width at context $x$
\begin{equation*}
 W_n(x)=\max_{a\in\cA}
 \max\left\{
 \overline r_n(x,a)-\underline r_n(x,a),
 \overline\Delta_n(x,a)-\underline\Delta_n(x,a)
 \right\}
\end{equation*}
can be computed from the data without knowing $V_\star$.

Given these intervals, we retain every action whose effect interval intersects $[-\tau,\tau]$ and choose the retained action with the largest reward upper endpoint. The intersection means that the data have not ruled out the action being fair. We call an interval \emph{valid} when it contains the corresponding true reward or effect. If all intervals are valid, the optimal fair action is retained. Optimistic reward selection then bounds regret by the selected action's reward interval width. A retained action can still violate the constraint, but its true effect and a value in $[-\tau,\tau]$ lie in the same interval, so its violation is at most the effect interval width. In the following lemma, we suppress the sample-size subscript $n$.

\begin{lemma}
\label{lem:screening}
Fix a context $x$ with $\cA_\tau(x)\ne\varnothing$ and suppose that every reward and effect interval is valid. Retain the actions whose effect intervals intersect $[-\tau,\tau]$, and let $A$ maximize the reward upper endpoint among the retained actions. Then $a^\star(x)$ is retained, and we have $r_{\theta^\star}(x,a^\star(x))-r_{\theta^\star}(x,A) \le \overline r(x,A)-\underline r(x,A)$ and $\pos{|\Delta_{\theta^\star}(x,A)|-\tau} \le \overline\Delta(x,A)-\underline\Delta(x,A)$.
\end{lemma}

Thus, valid intervals bound each of the one-round regret and violation by $W_n(x)$. We now choose the parameters that make these intervals valid with high probability. The minimum sample size $n_0$ ensures that the fitted MLE is finite and its Fisher matrix is invertible with high probability. The confidence factor $\ell$ accounts for all model fits and interval queries during the run. For $\delta\in(0,1)$, set $J=\lceil\log_2T\rceil$, $M_{\mathrm{dir}}=\min\{KT,5^d\}$, and $\ell=\log({32(J+1)(T+2)M_{\mathrm{dir}}}/{\delta})$. Here $J+1$ bounds the number of model fits in either algorithm. Set $n_0=\left\lceil{\kappa\Gamma_{\mathrm{corr}}}/{\sigma_0^2}\right\rceil$, where $\Gamma_{\mathrm{corr}}=\widetilde O(d)$ is the explicit threshold given in the appendix. In particular, $n_0=\widetilde O(\kappa d/\sigma_0^2)$.

\begin{lemma}
\label{lem:master-interface}
Suppose Assumptions~\ref{ass:bounded} and~\ref{ass:coverage} hold and $n_0<T$. Run either algorithm in Sections~\ref{sec:fixed-ete}--\ref{sec:adaptive-ete}. With probability at least $1-\delta$, every MLE computed by the algorithm is finite and its estimated Fisher matrix is invertible. Whenever the algorithm uses an estimate based on $n$ exploration observations to compute intervals at $X_t$, all factual reward and effect intervals are valid and
\begin{equation}
 W_n(X_t)\le1024\sqrt{\frac{V_\star\ell}{n}}.
 \label{eq:master-width}
\end{equation}
Both conclusions hold on the same event, simultaneously for all such fits and rounds.
\end{lemma}

The proof establishes finiteness of the MLE and combines Fisher concentration~\citep{tropp2015introduction} with a fixed-design logistic confidence bound~\citep{jun2021improved}. Applying $\mu$ and subtracting endpoints gives the reward and effect intervals. The width bound decreases as $n^{-1/2}$, with dependence on the reward and effect directions measured by $V_\star$. Lemma~\ref{lem:screening} then bounds per-round regret and violation during exploitation.

\subsection{Fixed-Length Explore--Then--Exploit}
\label{sec:fixed-ete}

Fix $m\in\{n_0,\ldots,T\}$ in advance; this does not require $V_\star$. Algorithm~\ref{alg:fixed-ete} explores uniformly for $m$ rounds, fits $(\wh\theta_m,\wh H_m)$, and uses this model with the optimistic screening rule from Section~\ref{sec:confidence} in all remaining rounds.

The exploration rounds contribute at most $m$ to each of regret and violation. Lemma~\ref{lem:master-interface} bounds the exploitation cost by $1024T\sqrt{V_\star\ell/m}$. Increasing $m$ raises the first cost and lowers the second; balancing them gives the $T^{2/3}$ rate. This tuning uses $V_\star$ and respects $n_0\le m\le T$.

\begin{algorithm}[H]
\caption{Fixed-length ETE}
\label{alg:fixed-ete}
\begin{algorithmic}[1]
\Require
$T,m,\delta,\tau,\phi,\phi^G,\kappa,\sigma_0$,
where $m\in\{n_0,\ldots,T\}$

\For{$t=1,\ldots,m$}
    \State Observe $X_t$, choose
    $A_t\sim\operatorname{Unif}(\mathcal A)$ independently,
    and observe $R_t$
\EndFor

\State Compute $(\widehat{\theta}_m,\widehat H_m)$
using Eqs.~\eqref{eq:mle}--\eqref{eq:empirical-fisher}

\For{$t=m+1,\ldots,T$}
    \State Observe $X_t$ and compute the reward and effect intervals
    \State $\displaystyle
    \widehat{\mathcal A}_t
    \gets
    \left\{
    a\in\mathcal A:
    [\underline\Delta_m(X_t,a),\overline\Delta_m(X_t,a)]
    \cap[-\tau,\tau]\neq\varnothing
    \right\}$
    \State $\displaystyle
    A_t\in
   {\arg\max}_{a\in\widehat{\mathcal A}_t}\,
    \overline r_m(X_t,a)$;
    observe $R_t$
\EndFor
\end{algorithmic}
\end{algorithm}

\begin{theorem}[Fixed-length ETE]
\label{thm:fixed}
Suppose Assumptions~\ref{ass:bounded}--\ref{ass:coverage} hold and $n_0<T$.  For every $m\in\{n_0,\ldots,T\}$, fixed-length ETE satisfies, with probability at least $1-\delta$,
\begin{equation}
 \max\{R_T,V_T\}
 \le
 m+1024T\sqrt{\frac{V_\star\ell}{m}}.
 \label{eq:fixed-bound}
\end{equation}Choosing $m
=\min\left\{ T, n_0+\lceil(1024^2V_\star\ell T^2)^{1/3}\rceil \right\}$ gives $\max\{R_T,V_T\} =O\!(n_0+(V_\star\ell)^{1/3}T^{2/3})$.
\end{theorem}

\begin{proof}[Proof Sketch]
The first $m$ rounds contribute at most $m$ to each of $R_T$ and $V_T$. On the event of Lemma~\ref{lem:master-interface}, Lemma~\ref{lem:screening} bounds the regret and violation in each subsequent round by the largest reward or effect interval width at $X_t$. Summing \eqref{eq:master-width} over the remaining rounds and using $T-m\le T$ gives \eqref{eq:fixed-bound}. Balancing the two terms yields the stated choice of $m$ and the $T^{2/3}$ rate.
\end{proof}

\subsection{Adaptive Explore--Then--Exploit}
\label{sec:adaptive-ete}

The choice of $m$ in Theorem~\ref{thm:fixed} requires $V_\star$, which is generally unknown. Without knowing $V_\star$, we can choose $m=\min\{T,n_0+\lceil T^{2/3}\rceil\}$ to obtain $\max\{R_T,V_T\}=O\!\left(n_0+(1+\sqrt{V_\star\ell})T^{2/3}\right)$ with probability at least $1-\delta$. Algorithm~\ref{alg:adaptive-ete} attains the $V_\star^{1/3}$ dependence without knowing $V_\star$ by checking the current interval widths. Narrow intervals allow exploitation; wide intervals trigger another exploration batch and a new fit, so exploration can resume after exploitation has begun.

Suppose $n_0<T$. We call each model update a \emph{checkpoint} and prescribe the cumulative exploration sample sizes $n_j=\min\{2^jn_0,T\}$, ending the sequence when it first reaches $T$. At checkpoint $j$, compute $(\wh\theta_{n_j},\wh H_{n_j})$ from $\mathcal D$, with $|\mathcal D|=n_j$, using Eqs.~\eqref{eq:mle}--\eqref{eq:empirical-fisher}. The algorithm keeps this estimate until another exploration batch is completed. For quantities computed from these observations, we abbreviate the subscript $n_j$ by $j$; in particular, $\wh\theta_j=\wh\theta_{n_j}$, $\wh H_j=\wh H_{n_j}$, and $W_j=W_{n_j}$. The sample size $n_j$ counts estimation observations, not elapsed rounds. If $n_0\ge T$, the algorithm instead explores uniformly for all $T$ rounds.

At checkpoint $j$, the algorithm exploits whenever $W_j(X_t)\le n_j/T$. The threshold compares the bound on regret and violation from the current intervals to the exploration cost $n_j$. When the intervals are valid, every round that passes this test contributes at most $n_j/T$ to each of regret and violation. Even $T$ such rounds at this checkpoint would thus contribute at most $n_j$ to each. If the test fails, we collect more exploration data up to the next checkpoint and refit, or explore until the horizon ends.

The triggering round is excluded from fitting because the decision to explore depends on its context. Keeping it would favor contexts with wide intervals in the estimation data, even though the action is drawn uniformly. We instead schedule the next exploration batch before seeing its contexts. The fitting data can then be viewed as prefixes of a single iid uniform-exploration sequence, which allows us to use the same confidence analysis at every checkpoint.

\begin{algorithm}[t]
\caption{Adaptive ETE}
\label{alg:adaptive-ete}
\begin{algorithmic}[1]
\Require $T,\delta,\tau,\phi,\phi^G,\kappa,\sigma_0$

\State Compute $\ell,\Gamma_{\mathrm{corr}},n_0$; define
$n_i=\min\{2^i n_0,T\}$ for $i\ge0$ and set $j\gets0$

\For{$t=1,\ldots,n_0$}
    \State Observe $X_t$, choose
    $A_t\sim\operatorname{Unif}(\mathcal A)$ independently,
    and observe $R_t$
\EndFor

\State Set $\mathcal D\gets((X_s,A_s,R_s))_{s=1}^{n_0}$; compute
$(\widehat\theta_j,\widehat H_j)$ from $\mathcal D$ using
Eqs.~\eqref{eq:mle}--\eqref{eq:empirical-fisher} with $n=n_j$
\State Set $t\gets n_0+1$

\While{$t\le T$}
    \State Observe $X_t$ and compute $W_j(X_t)$
    \If{$W_j(X_t)\le n_j/T$}
        \State Compute the reward and effect intervals at $X_t$ as in Section~\ref{sec:confidence}
        \State $\
        \widehat{\mathcal A}_t
        \gets
        \{
        a\in\mathcal A:
        [\underline\Delta_j(X_t,a),\overline\Delta_j(X_t,a)]
        \cap[-\tau,\tau]\neq\varnothing
        \}$
        \State $\displaystyle
        A_t\in
        {\arg\max}_{a\in\widehat{\mathcal A}_t}\,
        \overline r_j(X_t,a)$;
        observe $R_t$ and set $t\gets t+1$
    \Else
        \State Choose
        $A_t\sim\operatorname{Unif}(\mathcal A)$ and observe $R_t$,
        without updating $\mathcal D$
        \State Set $b\gets\min\{n_{j+1}-n_j,T-t\}$
        \State Before observing their contexts, schedule rounds
        $t+1,\ldots,t+b$ for the following loop
        \For{$s=t+1,\ldots,t+b$}
            \State Observe $X_s$, choose
            $A_s\sim\operatorname{Unif}(\mathcal A)$ independently
            \State
            Observe $R_s$ and append $(X_s,A_s,R_s)$ to $\mathcal D$
        \EndFor
        \If{$b=n_{j+1}-n_j$}
            \State Set $j\gets j+1$; compute $(\widehat\theta_j,\widehat H_j)$
            from $\mathcal D$ using Eqs.~\eqref{eq:mle}--\eqref{eq:empirical-fisher}
            with $n=n_j$
        \EndIf
        \State $t\gets t+b+1$
    \EndIf
\EndWhile
\end{algorithmic}
\end{algorithm}

All ties are resolved by a fixed deterministic rule. To bound the exploration cost, define $n_{\mathrm{crit}} =\max\{ n_0, (1024^2V_\star\ell T^2)^{1/3} \}$. This quantity is used only in the analysis; the algorithm checks the observable width instead. On the confidence event, once $n_j\ge n_{\mathrm{crit}}$, Lemma~\ref{lem:master-interface} implies $ W_j(X_t) \le1024\sqrt{{V_\star\ell}/{n_j}} \le\frac{n_j}{T}$, so no subsequent context can trigger further exploration. Doubling the checkpoint sizes limits any overshoot of this sufficient sample size to a factor of two, giving the same order of exploration cost as the fixed-length choice.

\begin{theorem}[Adaptive ETE]
\label{thm:adaptive}
Under Assumptions~\ref{ass:bounded}--\ref{ass:coverage}, with probability at least $1-\delta$, adaptive ETE satisfies $\max\{R_T,V_T\} \le\min\{T,4n_{\mathrm{crit}}+J+1\}$. Consequently,
\begin{equation*}
 \max\{R_T,V_T\}
 =\widetilde O\!\left(
 \bigl[V_\star\min\{\log K,d\}\bigr]^{1/3}T^{2/3}
 +\frac{\kappa d}{\sigma_0^2}
 \right)
\end{equation*}
where $\wt O$ suppresses logarithmic factors in $T$ and $1/\delta$, as well as constants.
\end{theorem}

\begin{proof}[Proof Sketch]
Every triggered batch targets $n_{j+1}\le2n_{\mathrm{crit}}$, since $n_j<n_{\mathrm{crit}}$ at a failed test. Thus at most $2n_{\mathrm{crit}}$ observations are retained, even if the last batch is truncated, and at most $J+1$ trigger rounds are discarded. Each exploitation round contributes at most $n_j/T\le2n_{\mathrm{crit}}/T$ to each of regret and violation. Summing these costs proves the bound.
\end{proof} 

\section{Experiments}
\label{sec:experiments}

We compare Fixed ETE-oracle (Algorithm~\ref{alg:fixed-ete}) and Adaptive ETE (Algorithm~\ref{alg:adaptive-ete}) with CCLB~\citep{chen2025causal}, DO-LogB~\citep{chen2026causal}, GLB-OMD~\citep{zhang2026generalized}, and the best-fair oracle. Fixed ETE-oracle uses $V_\star$ to choose its exploration length; we estimate $V_\star$ empirically using $\theta^\star$ and the full set of offered contexts. The best-fair oracle knows $\theta^\star$ and selects the optimal fair action. We report cumulative regret, cumulative stage-wise positive violations, and penalized regret as defined by \citet{chen2025causal}, with zero reward assigned to unfair actions. Regret can be negative because an unfair action may yield a higher reward than the best fair action. Each experiment is repeated $10$ times. Curves show means; shading spans two standard deviations on either side.

\paragraph{Synthetic experiment.} We construct a $d=10$ instance in which one coordinate is weakly represented in the factual features but remains visible in the counterfactual features. The coefficient of this coordinate in $\theta^\star$ is zero. Varying its factual scale therefore changes the information available for certifying counterfactual fairness without changing the mean rewards, fair action sets, or best fair actions. For the experiment in Figure~\ref{fig:main-experiments}\subref{fig:main-synthetic}, we set $K=5$, $T=10{,}000$, fairness threshold $\tau=0.05$, and factual scale $\rho=0.25$. The full experimental setup is given in Appendix~\ref{app:synthetic-experiment}.

\paragraph{Structural causal model.} We adapt the experimental setting of \citet{chen2025causal} to bounded logistic features. The structural causal model contains a binary sensitive attribute $A$, a covariate $W$, a mediator $M$ whose distribution depends on $(A,W)$, and an action-specific variable $D_a$. We obtain the counterfactual feature vector by replacing $A$ with $1-A$ and recomputing $M$ with the same exogenous noise. We clip and rescale features derived from $(A,W,M,D_a,MD_a)$ and draw Bernoulli rewards with a logistic mean. For the experiment in Figure~\ref{fig:main-experiments}\subref{fig:main-scm}, we set $K=20$, $T=10{,}000$, and $\tau=0.001$. Following \citet{chen2025causal}, we retain contexts for which at least $12$ actions are fair. The full experimental setup is given in Appendix~\ref{app:scm-experiment}.

\begin{figure}[t]
  \centering
  \begin{subfigure}{\textwidth}
    \centering
    \includegraphics[width=\linewidth,trim=0 14bp 0 0,clip]{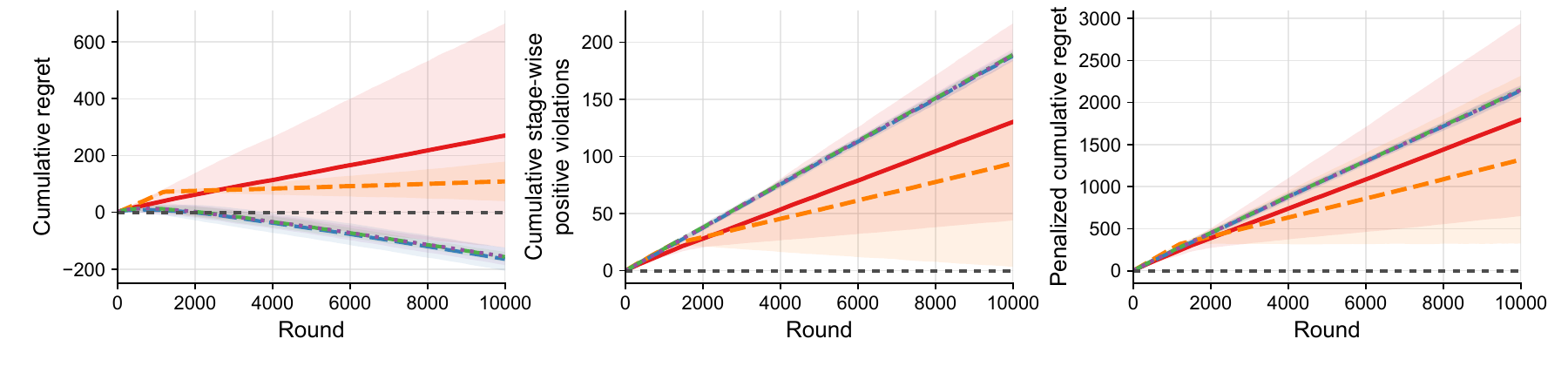}
    \caption{Synthetic experiment.}
    \label{fig:main-synthetic}
  \end{subfigure}\par
  \begin{subfigure}{\textwidth}
    \centering
    \includegraphics[width=\linewidth,trim=0 14bp 0 0,clip]{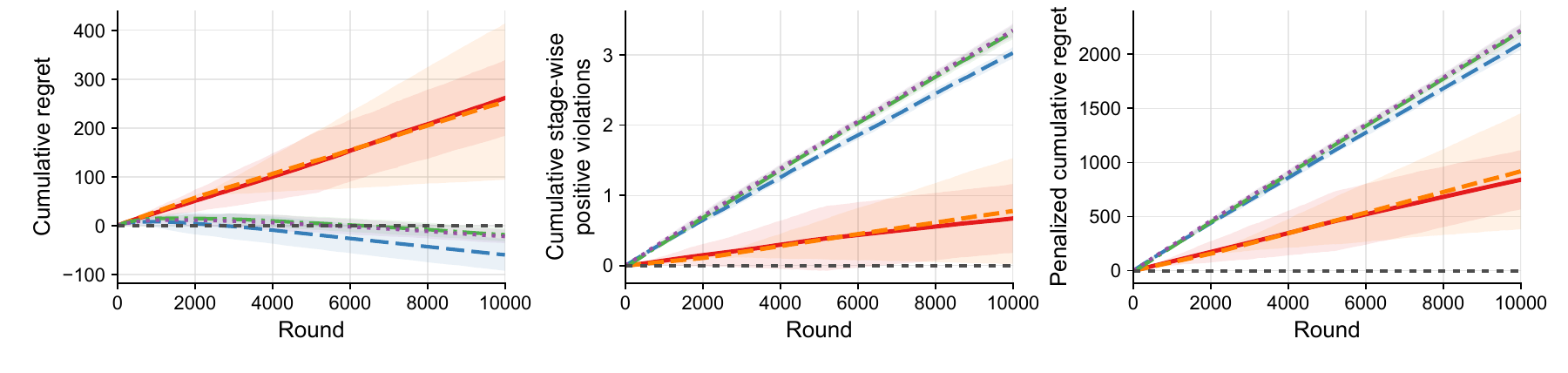}
    \par
    \includegraphics[width=0.96\linewidth]{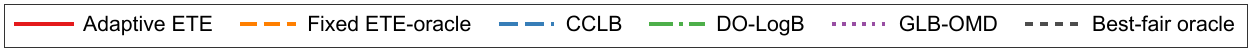}
    \par
    \caption{Structural causal model.}
    \label{fig:main-scm}
  \end{subfigure}
  \caption{Cumulative regret, cumulative stage-wise positive violations, and penalized cumulative regret for the proposed methods and baselines.}
  \label{fig:main-experiments}
\end{figure}

Figure~\ref{fig:main-experiments} shows that both ETE methods incur lower cumulative stage-wise positive violations and lower penalized regret than the learning baselines. Some baselines attain negative regret by choosing high-reward unfair actions, but incur higher cumulative stage-wise positive violations and higher penalized regret. Appendix~\ref{app:synthetic-experiment} presents additional experiments that vary the factual scale $\rho$, which controls the target-information scale $V_\star$. Appendix~\ref{app:scm-experiment} examines the effect of varying the fairness threshold $\tau$.

\bibliographystyle{plainnat}
\bibliography{ref}

\clearpage
\appendix

\section{Experimental Details and Additional Results}
\label{app:experiments}

\begin{figure}[H]
  \centering
  \begin{subfigure}{\textwidth}
    \centering
    \includegraphics[width=\linewidth,trim=0 14bp 0 0,clip]{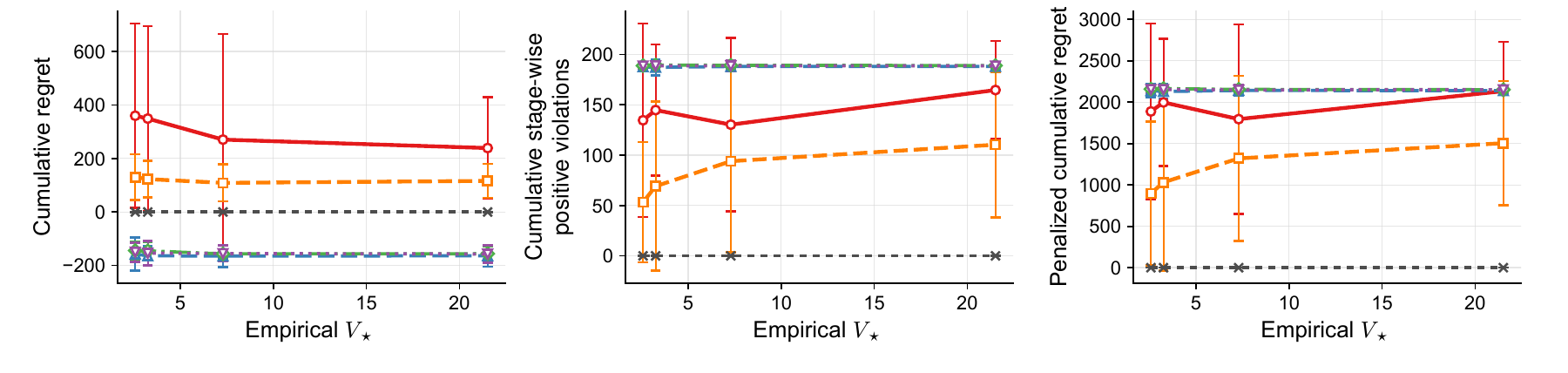}
    \caption{Synthetic experiment with varying factual scale $\rho$.}
    \label{fig:additional-synthetic}
  \end{subfigure}\par
  \begin{subfigure}{\textwidth}
    \centering
    \includegraphics[width=\linewidth,trim=0 14bp 0 0,clip]{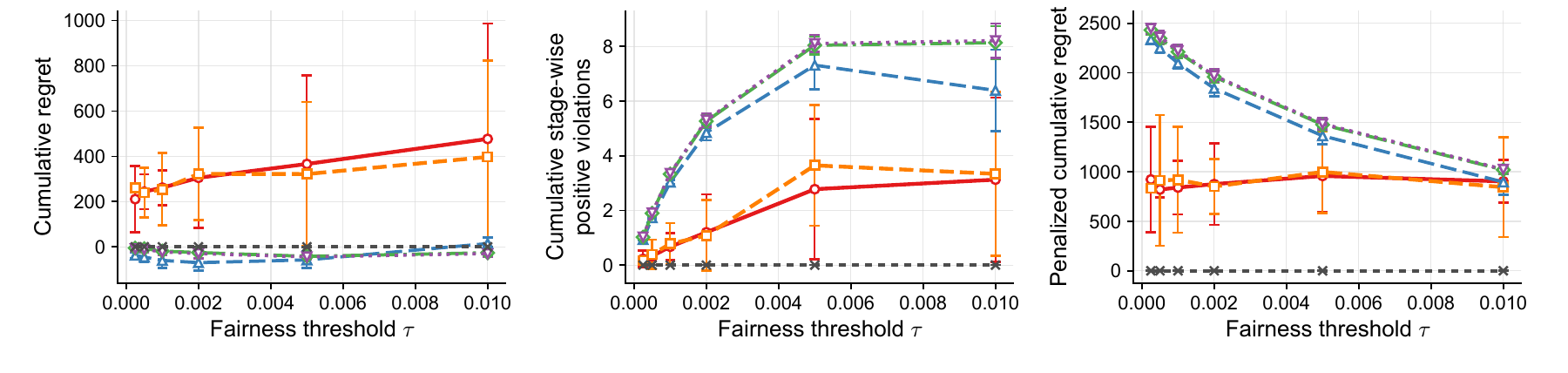}
    \par
    \includegraphics[width=0.96\linewidth]{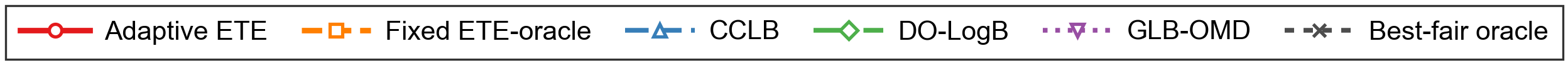}
    \par
    \caption{Structural causal model with varying fairness threshold $\tau$.}
    \label{fig:additional-scm}
  \end{subfigure}
  \caption{Comparison of cumulative regret, cumulative stage-wise positive violations, and penalized cumulative regret under varying $\rho$ and $\tau$.}
  \label{fig:additional-experiments}
\end{figure}

\subsection{Metrics and Implementation Details}
\label{app:evaluation-protocol}

We report three evaluation metrics and use the following implementation settings throughout. Curves show means over $10$ independent runs; shading and error bars span two standard deviations on either side.

Let $a^\star(X_t)$ be the best fair action. In addition to cumulative regret $R_t$ and cumulative stage-wise positive violations $V_t$ from Section~\ref{sec:model}, we report
\[
 R_t^{\mathrm{pen}}
 =\sum_{s=1}^t\left[r_{\theta^\star}(X_s,a^\star(X_s))
 -r_{\theta^\star}(X_s,A_s)\ind\{A_s\in\cA_\tau(X_s)\}\right].
\]
We use the penalized cumulative regret of \citet{chen2025causal}. In this metric, an unfair action receives zero reward. The best-fair oracle has zero regret, violation, and penalized regret.

The logistic MLE uses ridge $10^{-6}$. Both ETE algorithms begin with $100$ uniform exploration rounds and use a confidence radius coefficient of $0.01$. Adaptive ETE doubles the number of exploration samples used for estimation when its width test fails.

Fixed ETE-oracle sets its exploration length using an empirical version of $V_\star$ in Definition~\ref{def:vstar}. Using $\theta^\star$ and all $TK$ offered context--action pairs, define
\[
 \widehat I=\frac1{TK}\sum_{t=1}^T\sum_{a=1}^K
 \dot\mu(\phi(X_t,a)^\top\theta^\star)\phi(X_t,a)\phi(X_t,a)^\top
\]
and, with $q_\Delta=q_r-q_G$,
\[
 \widehat V_\star
 =\max_{t\in[T],a\in[K]}\max\left\{
 \normM{q_r(X_t,a)}{\widehat I^\dagger}^2,
 \normM{q_\Delta(X_t,a)}{\widehat I^\dagger}^2
 \right\}.
\]
The algorithm bounds the counterfactual effect by combining separate intervals for the factual and counterfactual rewards. The corresponding empirical scale is
\[
 \widehat V
 =\max_{t\in[T],a\in[K]}\max\left\{
 \normM{q_r(X_t,a)}{\widehat I^\dagger}^2,
 \left(\normM{q_r(X_t,a)}{\widehat I^\dagger}
 +\normM{q_G(X_t,a)}{\widehat I^\dagger}\right)^2
 \right\},
\]
where $\widehat I^\dagger$ is the Moore--Penrose inverse. The triangle inequality gives $\widehat V_\star\le\widehat V\le9\widehat V_\star$. Fixed ETE-oracle explores for $m=\max\{100,\lceil(\widehat V T^2)^{1/3}\rceil\}$ rounds. This choice omits the theoretical factor $1024^2\ell$; the initial exploration length and radius coefficient are also empirical choices. This benchmark uses $\theta^\star$ and future contexts and is not implementable online.

For CCLB, we use the primal--dual decision rule of \citet{chen2025causal} with ridge $1$ and a confidence radius coefficient of $2$. Its dual variable is projected onto $[0,10]$, and the model is refitted every $10$ rounds. For DO-LogB, we use the OMD update of \citet{chen2026causal}, replace its path-specific constraints with $|\Delta_\theta(x,a)|\le\tau$, and use $100$ initial uniform exploration rounds. DO-LogB and GLB-OMD~\citep{zhang2026generalized} use a confidence radius coefficient of $1/4$. GLB-OMD does not use the fairness constraint. An algorithm selects an action uniformly if its optimistic fair set is empty.

\subsection{Synthetic Experiment}
\label{app:synthetic-experiment}

This experiment varies the information available for estimating the counterfactual effect while keeping the rewards and fair action sets fixed. For every $t\in[T]$, $a\in[K]$, and $j\in[d]$, draw $S_{t,a,j}$ independently and uniformly from $\{-1,1\}$. Set $\phi_j(X_t,a)=S_{t,a,j}/\sqrt d$ for $j<d$ and $\phi_d(X_t,a)=\rho S_{t,a,d}/\sqrt d$. The counterfactual feature of action $1$ equals its factual feature, so this action is always fair. For each remaining action, the first and last coordinates of $\phi^G(X_t,a)$ are independent random permutations of $K-1$ evenly spaced points in $[-1,1]/\sqrt d$. The other coordinates equal those of $\phi(X_t,a)$. We use
\[
 d=10,\quad K=5,\quad T=10{,}000,\quad
 \theta^\star=(0.75,-0.45,0.30,-0.20,0,\ldots,0)^\top,
\]
and set $\rho=0.25$ and $\tau=0.05$ for Figure~\ref{fig:main-experiments}\subref{fig:main-synthetic}. Since $\theta_d^\star=0$, changing $\rho$ leaves the factual and counterfactual means, fair sets, and best-fair actions unchanged. It changes only the information available in the last coordinate. The mean number of fair actions is $3.00$.

Figure~\ref{fig:additional-experiments}\subref{fig:additional-synthetic} varies $\rho\in\{1,1/2,1/4,1/8\}$. The corresponding empirical values of $V_\star$, averaged over $10$ trials, are $2.58$, $3.25$, $7.29$, and $21.52$. Fixed ETE-oracle has larger violation and penalized regret as $V_\star$ increases. Adaptive ETE also has larger values when $V_\star$ is largest.

\subsection{Structural Causal Model}
\label{app:scm-experiment}

We adapt the structural causal model of \citet{chen2025causal} by clipping and rescaling the features and using a logistic reward model in place of their Gaussian-noise threshold model. Draw $A\sim\Bern(1/2)$ and $W\sim\mathcal N(0,1-A/2)$, where the second parameter of the normal distribution is the variance. Independently draw $U_M,Z_M\sim\mathcal N(0,1)$ and define
\[
 s(a,w,u)=
 \begin{cases}
 |w|/2+|u|/3,&a=1,\\
 |w|/3+|u|/2,&a=0.
 \end{cases}
\]
Using the same $U_M$ and $Z_M$ in the factual and counterfactual worlds, set
\[
 M=s(A,W,U_M)^{1/2}Z_M,\qquad
 M^G=s(1-A,W,U_M)^{1/2}Z_M.
\]
For each $a\in[20]$, draw $D_a\sim\mathcal N(0,\max\{|W|,|M|\})$. Let $[u]_{-1}^1=\max\{-1,\min\{u,1\}\}$. The factual feature and its counterfactual counterpart are
\[
 \phi(X,a)=\frac{1}{\sqrt5}\left(A,[W/2.5]_{-1}^1,[M/2.5]_{-1}^1,
 [D_a/3.5]_{-1}^1,[M/2.5]_{-1}^1[D_a/3.5]_{-1}^1\right)^\top,
\]
where $\phi^G(X,a)$ is obtained by replacing $(A,M)$ with $(1-A,M^G)$. We set
\[
 \theta^\star=(0,-\sqrt5/2,0,0,35\sqrt5/12)^\top
\]
and draw rewards from $\Bern(\mu(\phi(X,a)^\top\theta^\star))$. Coordinate-wise clipping and the factor $1/\sqrt5$ ensure that both feature vectors have norm at most one.

At each round, we use rejection sampling until at least $12$ actions satisfy $|\Delta_{\theta^\star}(X_t,a)|\le\tau$, following \citet{chen2025causal}. Figure~\ref{fig:main-experiments}\subref{fig:main-scm} uses $\tau=0.001$. The mean number of fair actions is $18.00$. Figure~\ref{fig:additional-experiments}\subref{fig:additional-scm} uses $\tau\in\{0.00025,0.0005,0.001,0.002,0.005,0.01\}$. Rejection sampling gives a different accepted context distribution for each value of $\tau$. For all $6$ values, one of the two ETE algorithms has the smallest mean penalized regret among the learning algorithms.

\clearpage
\section{Related Work}
\label{sec:additional-related-work}

\paragraph{Causal bandits and counterfactual fairness.} Early causal-bandit work uses known causal structure to share information across interventions, improving best-intervention identification \citep{lattimore2016causal,sen2017identifying} and cumulative regret \citep{lu2020regret}. Counterfactual fairness asks whether a prediction or outcome would change under a counterfactual intervention on a protected attribute \citep{kusner2017counterfactual}; causal and path-specific criteria further distinguish the mechanisms through which disparities arise \citep{zhang2018fairness,chiappa2019path,wu2019counterfactual,plecko2023causal}. \citet{huang2022achieving} bring counterfactual fairness to causal bandits using d-separation-based reward estimation and fair policy selection. \citet{lear2025causal} use causal decomposition to study long-term fairness in Markov decision processes. For logistic rewards, \citet{chen2025causal} study a counterfactual mean-reward constraint with a shared unknown parameter, while \citet{chen2026causal} consider path-specific fairness. In our setting, the causal structure is supplied through known factual and counterfactual feature maps. We study the resulting constrained bandit problem: whether factual reward observations provide enough information to learn counterfactual feasibility, and the regret--violation tradeoff when they do.

\paragraph{Bandits with constraints.} Stage-wise constrained bandits require feasibility at each round \citep{amani2019linear,khezeli2020safe,moradipari2020stage,amani2020generalized,pacchiano2021stochastic}. Long-term formulations instead bound cumulative constraint values, allowing strictly feasible rounds to offset violations, often through primal--dual updates \citep{jenatton2016adaptive,liu2021efficient,zhou2022kernelized}. The analysis of \citet{chen2025causal} uses a signed cumulative violation measure (Appendix~F, Eqs.~(82) and~(101)). We control the stronger measure of cumulative stage-wise positive violations, $V_T=\sum_{t=1}^T[|\Delta_{\theta^\star}(X_t,A_t)|-\tau]_+$, which counts each positive violation separately while allowing violations during learning. The learner receives no separate constraint feedback: both reward and fairness must be estimated from the same factual observations. We establish the need for coverage and obtain matching leading upper and worst-case lower rates under a full-rank condition on the factual covariance pooled across actions.

\paragraph{Logistic bandits.} Logistic bandits are a special case of generalized linear bandits \citep{filippi2010parametric,li2017provably}. Subsequent work sharpens their confidence sets, curvature dependence, and computational efficiency, and extends the model to neural reward functions \citep{faury2020improved,abeille2021instance,jun2021improved,faury2022jointly,lee2024unified,zhang2026generalized,bae2025neural}. These results estimate rewards at the available action features. In our problem, the same logistic parameter must also predict rewards at counterfactual features and the difference between the two means. The target scale \(V_\star\) is the largest inverse-Fisher variance over the reward and effect directions.

\paragraph{Coverage assumptions.} Full-rank context-covariance assumptions are used in stochastic linear contextual bandits and best-of-both-worlds analyses \citep{wu2020stochastic,bastani2021mostly,kuroki2024best}. Assumption~\ref{ass:coverage} requires the covariance generated by uniform action sampling to be nonsingular but allows the covariance of each action to be singular. Experimental-design methods choose a sampling distribution to reduce uncertainty in the directions that distinguish promising actions \citep{fiez2019sequential,wagenmaker2021experimental,mason2022experimental}. Our target scale \(V_\star\) measures the largest uncertainty of a reward or effect direction under the Fisher information generated by uniform exploration. \citet{hao2020high} also obtain a \(T^{2/3}\) rate by balancing costly exploration against estimation error. The same balance determines the exploration length in our upper bounds.

\section{Proof of Proposition~\ref{prop:no-coverage}}
\label{app:no-coverage-proof}

Take $d=2$, $\cA=\{1,2\}$, and $X_t=x_0$ almost surely for every $t$. Define
\[
 \theta_0=
 \begin{pmatrix}
 2\log 3\\
 0
 \end{pmatrix},
 \qquad
 \theta_1=
 \begin{pmatrix}
 2\log 3\\
 -2\log 21
 \end{pmatrix},
 \qquad
 \Theta=\{(1-\lambda)\theta_0+\lambda\theta_1:\lambda\in[0,1]\}.
\]
Set $\tau=1/8$. With $e_1,e_2$ denoting the standard basis of $\R^2$, choose
\[
 \phi(x_0,1)=\frac12e_1,
 \qquad
 \phi^G(x_0,1)=\frac12(e_1+e_2),
 \qquad
 \phi(x_0,2)=\phi^G(x_0,2)=0.
\]
In environment $j\in\{0,1\}$, the true parameter is $\theta^\star=\theta_j$. The context distribution, feature maps, parameter set, and threshold are the same in both environments. The factual reward means satisfy
\[
 r_{\theta_0}(x_0,1)=r_{\theta_1}(x_0,1)=\mu(\log 3)=\frac34,
 \qquad
 r_{\theta_j}(x_0,2)=r_{\theta_j}^G(x_0,2)=\frac12.
\]
The counterfactual reward means of action $1$ are
\[
 r_{\theta_0}^G(x_0,1)=\mu(\log 3)=\frac34,
 \qquad
 r_{\theta_1}^G(x_0,1)=\mu(-\log 7)=\frac18.
\]
Thus $\Delta_{\theta_0}(x_0,1)=0$, $\Delta_{\theta_1}(x_0,1)=5/8$, and $\Delta_{\theta_j}(x_0,2)=0$. Both actions are fair under $\theta_0$, so the fair oracle chooses action $1$. Only action $2$ is fair under $\theta_1$. The one-round $(\text{regret},\text{violation})$ pairs are
\[
\begin{array}{c|cc}
 & \text{action }1 & \text{action }2\\
\hline
 \theta_0 & (0,0) & (1/4,0)\\
 \theta_1 & (-1/4,1/2) & (0,0).
\end{array}
\]

Fix any adaptive, possibly randomized policy $\pi$, and let
\[
 N_1(T)=\sum_{t=1}^T\ind\{A_t=1\}.
\]
Every action has the same Bernoulli reward law under $\theta_0$ and $\theta_1$. Since the context and observed feature maps also agree, $\pi$ induces the same distribution of observed histories, and hence of $N_1(T)$, in both environments. Write $m=\E_{0,\pi}[N_1(T)]=\E_{1,\pi}[N_1(T)]$. The table gives
\[
 \E_{0,\pi}[R_T]=\frac{T-m}{4},
 \qquad
 \E_{1,\pi}[V_T]=\frac{m}{2}.
\]
If $m\le T/3$, then $\E_{0,\pi}[R_T]\ge T/6$; if $m\ge T/3$, then $\E_{1,\pi}[V_T]\ge T/6$. This proves the first claim.

Action $1$ under $\theta_1$ has regret $-1/4$ and violation $1/2$, so its joint loss is $1/4$. Consequently,
\[
 \E_{0,\pi}[R_T+V_T]=\frac{T-m}{4},
 \qquad
 \E_{1,\pi}[R_T+V_T]=\frac{m}{4}.
\]
The maximum of these two quantities is at least $T/8$, which proves the second claim.

The construction satisfies Assumptions~\ref{ass:bounded} and~\ref{ass:fair}. The set $\Theta$ is convex and compact, with $\norm{\theta}_2\le\norm{\theta_1}_2=2\sqrt{(\log3)^2+(\log21)^2}$ for every $\theta\in\Theta$; we take this value as the parameter bound $S$. All feature norms are at most one, and action $2$ has zero effect for every $\theta\in\Theta$. For $\theta=(1-\lambda)\theta_0+\lambda\theta_1$, the logits of action $1$ are
\[
 \phi(x_0,1)^\top\theta=\log 3,
 \qquad
 \phi^G(x_0,1)^\top\theta
 =\log 3-\lambda\log 21
 \in[-\log 7,\log 3].
\]
The logits of action $2$ are zero. Therefore $\kappa=64/7<\infty$. Finally, the pooled factual covariance is
\[
 \Sigma
 =
 \frac12\phi(x_0,1)\phi(x_0,1)^\top
 +\frac12\phi(x_0,2)\phi(x_0,2)^\top
 =
 \frac18e_1e_1^\top.
\]
Thus $\theta_1-\theta_0\in\ker(\Sigma)$, while $\phi^G(x_0,1)^\top(\theta_1-\theta_0)=-\log21\ne0$. The coverage assumption excludes this construction.

\section{Proof of Theorem~\ref{thm:binary}}
\label{app:lower-notation}

We use the following notation for both lower-bound proofs. In environment $j$, with true parameter $\theta_j$, write
\[
  r_j=r_{\theta_j},\qquad r_j^G=r_{\theta_j}^G,
  \qquad \Delta_j=\Delta_{\theta_j},\qquad I_j=I(\theta_j).
\]
Let $\cA_{\tau,j}(x)=\{a\in\cA:|\Delta_j(x,a)|\le\tau\}$ and $a_j^\star(x)\in\arg\max_{a\in\cA_{\tau,j}(x)}r_j(x,a)$. Define the one-round joint cost by
\[
  \mathrm{cost}_j(x,a)
  =r_j(x,a_j^\star(x))-r_j(x,a)
  +\pos{|\Delta_j(x,a)|-\tau}.
\]
Then $R_T+V_T=\sum_{t=1}^T\mathrm{cost}_j(X_t,A_t)$ in environment $j$.

\subsection{Construction of the Worst-Case Instance}

Set the constants
\begin{equation*}
  v_0=\dot\mu(3/4),
  \qquad
  \Delta_E=\frac12-\mu(-3/4)-\frac1{128}>0.
\end{equation*}
Indeed, $\mu(-3/4)<1/3$, so $\Delta_E>1/6-1/128>1/128$.

Fix $B\ge1$, $\rho\in(0,1]$, and $\alpha\in(0,1/4]$.  Define
\begin{equation}
  v_B=\dot\mu(B),
  \qquad
  g_\alpha=\frac{v_0\alpha}{8},
  \qquad
  b_\alpha=\log\frac{1/2+g_\alpha}{1/2-g_\alpha}.
  \label{eq:binary-scalars}
\end{equation}
The context $X$ is uniform on $\{-1,1\}$. Define the two parameters by
\begin{equation*}
  \theta_0=(B\sqrt{1+\rho^2},0)^\top,
  \qquad
  \theta_1=(B\sqrt{1+\rho^2},\alpha)^\top,
  \qquad
  \Theta=\operatorname{conv}\{\theta_0,\theta_1\}.
\end{equation*}
There are three actions: a high-reward action $H$, an always-fair baseline $S$, and an informative action $E$. Their features and roles are

\begingroup
\renewcommand{\arraystretch}{1.35}

\begin{equation}
\begin{array}{c|c|c|c}
\text{action}
  & \phi(x,a)
  & \phi^G(x,a)
  & \text{role}
\\ \hline
H
  & 0
  & e_2
  & \text{high reward but unknown fairness}
\\[3mm]
S
  & -\dfrac{b_\alpha}{B\sqrt{1+\rho^2}}e_1
  & -\dfrac{b_\alpha}{B\sqrt{1+\rho^2}}e_1
  & \text{fair baseline}
\\[5mm]
E
  & \dfrac{-e_1+x\rho e_2}{\sqrt{1+\rho^2}}
  & \dfrac{-e_1+x\rho e_2}{\sqrt{1+\rho^2}}
  & \text{informative but lower reward}
\end{array}
\label{eq:binary-features}
\end{equation}

\endgroup

Finally, set the fairness threshold from Section~\ref{sec:model} to
\begin{equation*}
  \Delta_\alpha=\mu(\alpha)-\frac12,
  \qquad
  \tau=\tau_\alpha=\frac{\Delta_\alpha}{2}.
\end{equation*}
The context distribution, feature maps, parameter set, and threshold are common and known to the learner in both environments. Only the true parameter differs.

The next lemma verifies the assumptions and gives the joint cost of each action in both environments.

\begin{lemma}
\label{lem:binary-costs}
For every $B\ge1$, $\rho\in(0,1]$, $0<\alpha\le1/4$, and $x\in\{-1,1\}$, the construction satisfies Assumptions~\ref{ass:bounded}, \ref{ass:fair}, and~\ref{ass:coverage}. Its one-round joint costs are
\begin{equation*}
\begin{array}{c|ccc|c}
 &\mathrm{cost}_j(x,H)&\mathrm{cost}_j(x,S)&\mathrm{cost}_j(x,E)&a_j^\star(x)\\ \hline
j=0&0&g_\alpha&\ge\Delta_E&H\\
j=1&\ge3g_\alpha&0&\ge\Delta_E&S.
\end{array}
\end{equation*}
In particular, every action other than $H$ costs at least $g_\alpha$ in environment $0$, and the informative action costs at least $\Delta_E$ in both environments.
\end{lemma}

\begin{proof}
First note that $g_\alpha\le1/128$ and
\[
  b_\alpha=2\operatorname{arctanh}(2g_\alpha)
  \le\log(65/63)<1/16.
\]
Hence all factual and counterfactual feature norms are at most one, and $\Theta$ is convex and compact. Also, the definition of $b_\alpha$ gives
\begin{equation*}
  \mu(-b_\alpha)=\frac12-g_\alpha.
\end{equation*}
Thus $H$ and $S$ have factual rewards $1/2$ and $1/2-g_\alpha$ in both environments. Since $\dot\mu$ decreases on $[0,\infty)$ and $\alpha\le1/4<3/4$,
\begin{equation*}
  \Delta_\alpha
  =\int_0^\alpha\dot\mu(z)\,dz
  \ge v_0\alpha
  =8g_\alpha.
\end{equation*}
Therefore $H$ is fair in environment $0$.  In environment $1$, its effect has absolute value $\Delta_\alpha$, so its violation is $\Delta_\alpha-\tau_\alpha=\Delta_\alpha/2$.  Its regret relative to $S$ is $-g_\alpha$, and consequently
\begin{equation*}
  \mathrm{cost}_1(x,H)
  =-g_\alpha+\frac{\Delta_\alpha}{2}
  \ge3g_\alpha.
\end{equation*}
Actions $S$ and $E$ have identical factual and counterfactual features, so their effects are zero and they are fair in both environments.

It remains to verify the cost of action $E$ in both environments. Its factual logit is $-B$ in environment $0$, while in environment $1$ it is
\[
  -B+\frac{x\rho\alpha}{\sqrt{1+\rho^2}}
  \le -B+\alpha\le-3/4.
\]
Hence
\begin{gather}
  \mathrm{cost}_0(x,E)
  =\frac12-\mu(-B)
    \ge\frac12-\mu(-3/4)>\Delta_E,
    \label{eq:binary-E-cost-0}\\
  \mathrm{cost}_1(x,E)
  =\frac12-g_\alpha
    -\mu\!\left(-B+\frac{x\rho\alpha}{\sqrt{1+\rho^2}}\right)
    \ge\frac12-\frac1{128}-\mu(-3/4)
    =\Delta_E.
    \label{eq:binary-E-cost-1}
\end{gather}
Since $\Delta_E>1/128\ge g_\alpha$, these calculations prove the cost table and verify that the fair set is nonempty. Finally,
\begin{equation}
  \Sigma =\E\bigl[\phi(X,A)\phi(X,A)^\top\bigr]
  =\frac13
  \begin{pmatrix}
    \dfrac{1+b_\alpha^2/B^2}{1+\rho^2}&0\\[2mm]
    0&\dfrac{\rho^2}{1+\rho^2}
  \end{pmatrix},
  \quad
  \sigma_0^2=\frac{\rho^2}{3(1+\rho^2)}>0.
  \label{eq:binary-covariance}
\end{equation}
The first diagonal entry is no smaller than the second because $\rho\le1$.  Thus the displayed value is exactly the minimum eigenvalue, and Assumption~\ref{ass:coverage} holds.
\end{proof}

We next bound $V_\star$ for this construction. Recall that
\[
  C=\frac{1+\rho^2}{4e^{1/4}\rho^2v_B}.
\]

\begin{lemma}
\label{lem:binary-fisher}
At both true parameters $j\in\{0,1\}$,
\begin{equation}
  \frac C3\le V_{\star,j}\le3C.
  \label{eq:binary-vstar}
\end{equation}
Moreover, for the global curvature $\kappa$ of this common family,
\begin{equation}
  \frac{e^{-1/2}}{12}\frac{\kappa}{\sigma_0^2}
  \le C\le
  \frac{e^{-1/4}}{12}\frac{\kappa}{\sigma_0^2}.
  \label{eq:binary-curvature}
\end{equation}
\end{lemma}

\begin{proof}
Let
\[
  \Sigma_E=\frac1{1+\rho^2} \begin{pmatrix}
      1 & 0 \\ 0 & \rho^2
  \end{pmatrix}
\]
and let $I_E(\theta)$ be the Fisher information from action $E$, averaged over $X$. For every $\theta\in\Theta$, the logit of action $E$ lies between $-B-1/4$ and $-B+1/4$. Applying \eqref{eq:logistic-comparison} before taking expectation gives
\begin{equation}
  e^{-1/4}v_B\Sigma_E
  \preceq I_E(\theta)
  \preceq e^{1/4}v_B\Sigma_E.
  \label{eq:binary-IE}
\end{equation}
Because the factual feature of $H$ is zero, the Fisher matrix under uniform action sampling is
\begin{equation*}
  I(\theta)=\frac13\left(
    I_E(\theta)
    +\dot\mu(-b_\alpha)\phi(S)\phi(S)^\top
  \right).
\end{equation*}

We have $q_r(H;\theta)=0$ and $q_\Delta(H;\theta)=-\dot\mu(\theta_2)e_2$ at action $H$.  The term for action $S$ has no $e_2$ component. Restricting \eqref{eq:variational} to $u=se_2$ and using the upper side of \eqref{eq:binary-IE} gives
\begin{align*}
  q_\Delta(H;\theta)^\top I(\theta)^{-1}q_\Delta(H;\theta)
  &\ge
  \frac{3e^{-1/4}\dot\mu(\theta_2)^2(1+\rho^2)}{v_B\rho^2}
  \nonumber\\
  &=12\dot\mu(\theta_2)^2C.
\end{align*}

At either true parameter the second component is 0 or $\alpha$, so $\dot\mu(\theta_2)\ge v_0$. To check the constant, put $y=e^{3/4}\in(1,3)$; then $v_0=y/(1+y)^2>3/16$ because $3y^2-10y+3=3(y-1/3)(y-3)<0$. Hence $12v_0^2>1/3$, proving $V_{\star,j}\ge C/3$.

For $V_{\star,j}\le 3C$, the lower side of \eqref{eq:binary-IE} implies
\begin{equation*}
  I(\theta)^{-1}
  \preceq\frac{3e^{1/4}}{v_B}\Sigma_E^{-1}.
\end{equation*}
We now check each nonzero term in the definition of $V_\star$.
\begin{itemize}
\item For the effect of $H$, we have $\dot\mu\le1/4$, so
\[
  \norm[I(\theta)^{-1}]{q_\Delta(H;\theta)}^2
  \le\frac{3e^{1/4}}{16v_B}\frac{1+\rho^2}{\rho^2}
  <\frac54C.
\]
\item The effect of $S$ is zero.  Since $\phi(S)^\top\Sigma_E^{-1}\phi(S)=b_\alpha^2/B^2\le1/16$,
\[
  \norm[I(\theta)^{-1}]{q_r(S;\theta)}^2
  \le\frac{3e^{1/4}}{256v_B}
  \le\frac C{16}.
\]
\item The effect of $E$ is zero.  Moreover, $\phi(x,E)^\top\Sigma_E^{-1}\phi(x,E)=2$ and $\dot\mu(\phi(x,E)^\top\theta)\le e^{1/4}v_B$.  Hence
\[
  \norm[I(\theta)^{-1}]{q_r(x,E;\theta)}^2
  \le6e^{3/4}v_B
  \le\frac43C.
\]
\end{itemize}
For the last two comparisons, use $v_B=\dot\mu(B)\le\dot\mu(1)<1/5$ and $\rho^2/(1+\rho^2)\le1/2$. Thus $3C$ bounds all terms and proves the upper bound in \eqref{eq:binary-vstar}.

To bound the curvature, note that every factual and counterfactual logit has absolute value at most $B+1/4$, and action $E$ attains a logit of magnitude $B$. Since $\dot\mu$ is even and decreasing on $[0,\infty)$, \eqref{eq:logistic-comparison} gives
\begin{equation}
  v_B^{-1}\le\kappa\le e^{1/4}v_B^{-1}.
  \label{eq:binary-kappa-sandwich}
\end{equation}
Combining \eqref{eq:binary-covariance}, the definition $C=(1+\rho^2)/(4e^{1/4}\rho^2v_B)$, and \eqref{eq:binary-kappa-sandwich} gives
\[
  \frac{C}{\kappa/\sigma_0^2}
  =\frac{1}{12e^{1/4}\kappa v_B}
  \in\left[\frac{e^{-1/2}}{12},\frac{e^{-1/4}}{12}\right],
\]
which is \eqref{eq:binary-curvature}.
\end{proof}

\subsection{Main Proof of Theorem~\ref{thm:binary}}

We first bound the KL divergence between observations in the two environments.

\begin{lemma}
\label{lem:binary-kl}
For every context $x\in\{-1,1\}$,
\begin{equation*}
  \operatorname{KL}\!\left(
    P_0(\,\cdot\mid x,E)
    \,\middle\|\,
    P_1(\,\cdot\mid x,E)
  \right)
  \le\frac{\alpha^2}{8C}.
\end{equation*}
The factual reward distributions of $H$ and $S$ are identical in the two environments.
\end{lemma}

\begin{proof}
The two logits of $E$, conditional on $x$, are
\[
  z_0=-B,
  \qquad
  z_1=-B+\frac{x\rho\alpha}{\sqrt{1+\rho^2}}.
\]
Their squared difference is $\rho^2\alpha^2/(1+\rho^2)$, and the entire segment between $z_0$ and $z_1$ is within the interval $[-B-1/4, -B+1/4]$.  Lemma~\ref{lem:one-step-kl} and \eqref{eq:logistic-comparison} therefore give
\[
  \operatorname{KL}(P_0\|P_1)
  \le\frac{e^{1/4}v_B\rho^2\alpha^2}{2(1+\rho^2)}
  =\frac{\alpha^2}{8C}.
\]
The remaining part of the claim follows from \eqref{eq:binary-features}.
\end{proof}

The constants in Theorem~\ref{thm:binary} are
\[
  c_{\rm test}=\frac{e^{-1/8}v_0}{32},
  \qquad
  C_0=\frac{64\Delta_E}{c_{\rm test}},
  \qquad
  c_1=\Delta_E^{1/3}c_{\rm test}^{2/3}.
\]

\begin{proof}[Proof of Theorem~\ref{thm:binary}]
For a given $T$, choose
\[
  \alpha=\left(\frac{\Delta_EC}{c_{\rm test}T}\right)^{1/3}.
\]
Since $T\ge C_0C$ and $C_0=64\Delta_E/c_{\rm test}$, we have $\alpha^3\le1/64$, hence $\alpha\le1/4$, so the construction is valid. Fix an arbitrary adaptive, possibly randomized policy $\pi$ for the rest of the proof.

Let $N_E(t)=\sum_{s=1}^t\ind\{A_s=E\}$. We consider two cases according to how often the policy selects the informative action.

\begin{itemize}
\item \textbf{Case 1: action $E$ is used often.} If $\E_jN_E(T)\ge C/\alpha^2$ for one of $j\in\{0,1\}$, then the bounds \eqref{eq:binary-E-cost-0}--\eqref{eq:binary-E-cost-1} in that same environment give
    \begin{equation}
    \max_{j\in\{0,1\}}\E_j[R_T+V_T]
    \ge\frac{\Delta_EC}{\alpha^2}.
    \label{eq:binary-exploration-branch}
    \end{equation}
\item \textbf{Case 2: $E$ is rare in both environments.}  Suppose that $\E_jN_E(T)<C/\alpha^2$ for both $j$. By Lemmas~\ref{lem:adaptive-kl} and~\ref{lem:binary-kl}, the histories immediately before round $t$ satisfy
    \begin{equation*}
    \operatorname{KL}(P_0^t\|P_1^t)
    \le\frac{\alpha^2}{8C}\E_0N_E(t-1)
    <\frac18.
    \end{equation*}
By the data-processing statement in Lemma~\ref{lem:info-tools}, the same KL bound holds for the distributions of $A_t$. Applying \eqref{eq:bh} to the event $\{A_t=H\}$ yields
    \begin{equation}
    \Pp_0(A_t\ne H)+\Pp_1(A_t=H)
    \ge\frac12e^{-1/8}.
    \label{eq:binary-testing}
    \end{equation}
In environment $0$, every action other than $H$ costs at least $g_\alpha$. In environment $1$, action $H$ costs at least $3g_\alpha$, hence at least $g_\alpha$. Since the maximum is at least the average, summing \eqref{eq:binary-testing} over $t$ gives
    \begin{align}
    \max_{j\in\{0,1\}}\E_j[R_T+V_T]
    &\ge\frac12\sum_{t=1}^T
    \left[
    g_\alpha\Pp_0(A_t\ne H)
    +g_\alpha\Pp_1(A_t=H)
    \right]\nonumber\\
    &\ge\frac{e^{-1/8}g_\alpha T}{4}
    =c_{\rm test}\alpha T.
    \label{eq:binary-confusion-branch}
    \end{align}
\end{itemize}
Finally, the choice of $\alpha$ makes the right-hand sides of \eqref{eq:binary-exploration-branch} and \eqref{eq:binary-confusion-branch} equal:
\[
  \frac{\Delta_EC}{\alpha^2}
  =c_{\rm test}\alpha T
  =\Delta_E^{1/3}c_{\rm test}^{2/3}C^{1/3}T^{2/3}.
\] 
This proves \eqref{eq:binary-lower}. The comparison $V_{\star,j}\asymp C$ follows from Lemma~\ref{lem:binary-fisher}.
\end{proof}

\section{Proof of Theorem~\ref{thm:general}}

We use the environment-specific notation and one-round joint cost defined in Appendix~\ref{app:lower-notation}.

\subsection{Construction of the Worst-Case Instance}
\label{app:general-construction}

Fix $B\ge1$, $\rho\in(0,1]$, $d\ge46$, and $K\ge8$, and define
\begin{equation}
  n=d-1,
  \qquad
  M=\min\left\{
    \left\lfloor\frac K2\right\rfloor,
    \left\lfloor e^{n/32}\right\rfloor
  \right\},
  \qquad
  p=\frac{K-M-1}{K}.
  \label{eq:general-size}
\end{equation}
Here $v_B=\dot\mu(B)$ is the same notation as in \eqref{eq:binary-scalars}. Define the scale
\begin{equation}
  D=\frac{n(1+\rho^2)}{\rho^2v_B}.
  \label{eq:D}
\end{equation}
Since $n\ge45$ and $K\ge8$, we have $M\ge4$ and $3/8\le p<1$. Both terms defining $M$ are at least 4, and $M \le K/2$ implies $p\ge1/2-1/K\ge3/8$. Choose a hypercube packing $u_1,\ldots,u_M\in\{-1/\sqrt n,1/\sqrt n\}^n$ satisfying $u_i^\top u_j\le1/2$ for $i\ne j$, whose existence is guaranteed by Lemma~\ref{lem:packing}. For $0<\alpha\le1/4$, define
\begin{equation}
  g_\alpha=\frac{v_0\alpha}{8},
  \qquad
  b_\alpha=\log\frac{1/2+g_\alpha}{1/2-g_\alpha},
  \qquad
  c_\alpha=\left(
    1+\frac{\alpha^2}{B^2(1+\rho^2)}
  \right)^{-1/2}.
  \label{eq:general-scalars}
\end{equation}
Environment $j\in[M]$ has parameter
\begin{equation*}
  \theta_j=(B\sqrt{1+\rho^2},-\alpha u_j)^\top,
  \qquad
  \Theta=\operatorname{conv}\{\theta_1,\ldots,\theta_{M}\}.
\end{equation*}

The context $X=(L,\xi)$ is uniform on $[n]\times\{-1,1\}$.  There are $M$ high-reward candidate actions $H_1,\ldots,H_{M}$, one safe action $S_0$, and $K-M-1$ informative actions (replicas) $E_1,\ldots,E_{K-M-1}$. Their factual features are
\begin{gather*}
  \phi(x,H_i)=0,\nonumber\\
  \phi(x,S_0)=-\frac{b_\alpha}{B\sqrt{1+\rho^2}}e_1,
  \nonumber\\
  \phi((L,\xi),E_s)
  =\frac{-e_1+\xi\rho e_{L+1}}{\sqrt{1+\rho^2}}.
\end{gather*}
The counterfactual features of $S_0$ and every $E_s$ equal their factual features.  For a high action,
\begin{equation}
  \phi^G(x,H_i)
  =c_\alpha\left(
    \frac{\alpha}{B\sqrt{1+\rho^2}}e_1+(0,u_i)^\top
  \right).
  \label{eq:general-counterfactual}
\end{equation}
The choice of $c_\alpha$ in \eqref{eq:general-scalars} makes $\norm{\phi^G(x,H_i)}_2=1$. Finally choose the fairness threshold as
\begin{equation*}
  \tau=\tau_\alpha=\frac{v_0\alpha}{8}.
\end{equation*}

\begin{lemma}
\label{Lem:general-costs}
For every $B\ge1$, $\rho\in(0,1]$, $d\ge46$, $K\ge8$, $0<\alpha\le1/4$, every environment $j\in[M]$, every context $x=(L,\xi)$, and every informative action $E_s$, the following hold.
\begin{enumerate}
\item The construction satisfies Assumptions~\ref{ass:bounded}, \ref{ass:fair}, and~\ref{ass:coverage}.
\item $H_j$ is the unique fair action among high-reward actions $H_1, \ldots, H_M$. Hence $H_j$ is the fair oracle.
\item Every action $a$ that is not $H_j$ has $\mathrm{cost}_j(x,a)\ge v_0\alpha/8$.
\item Every informative action satisfies $\mathrm{cost}_j(x,E_s)\ge\Delta_E$.
\item $\Delta_j(x,a)=0$ and $q_\Delta(x,a;\theta)=0$ for every safe or informative action $a$ and every $\theta\in\Theta$.
\end{enumerate}
\end{lemma}

\begin{proof}
Item 5 is immediate because $\phi^G(x,a)=\phi(x,a)$ for $a\in\{S_0,E_1,\ldots,E_{K-M-1}\}$.

The factual reward of every high-reward action is $1/2$.  At $H_i$ in environment $j$, we have
\begin{equation*}
  z_{ij}
  =\phi^G(x,H_i)^\top\theta_j
  =c_\alpha\alpha(1-u_i^\top u_j).
\end{equation*}
For $i=j$, $z_{jj}=0$, so $H_j$ has zero effect.  If $i\ne j$, the construction gives $1-u_i^\top u_j\ge1/2$.  Moreover,
\begin{equation}
  c_\alpha^2
  \ge\left(1+\frac1{16}\right)^{-1}
  =\frac{16}{17}, 
  \quad \text{hence } c_\alpha>\frac34.
  \label{eq:calpha-lower}
\end{equation}
Since every inner product is at least $-1$, $0\le z_{ij}\le2\alpha\le1/2$.  Thus $\dot\mu(z)\ge v_0$ along $[0,z_{ij}]$. Also for $i\ne j$,
\begin{equation*}
  |\Delta_j(x,H_i)|
  =\mu(z_{ij})-\frac12
  \ge v_0z_{ij}
  \ge\frac{3v_0\alpha}{8}.
\end{equation*}
After subtracting the threshold $v_0\alpha/8$, every wrong high-reward action has violation at least $v_0\alpha/4$. The safe action has reward $1/2-g_\alpha$, zero effect, and hence joint cost $g_\alpha=v_0\alpha/8$.  An informative action has logit
\begin{equation*}
  -B-\frac{\xi\rho\alpha u_j(L)}{\sqrt{1+\rho^2}}
  \le-B+\alpha\le-3/4.
\end{equation*}
Its effect is zero, so its joint cost relative to $H_j$ is at least $1/2-\mu(-3/4)>\Delta_E$. Since $\Delta_E>1/128\ge g_\alpha$, this proves items 2, 3, and 4.

It remains to verify the assumptions. The parameter set is convex and compact, with radius at most $\sqrt2B+1/4$. The bound $b_\alpha\le\log(65/63)<1/16$, proved in Lemma~\ref{lem:binary-costs}, and the normalization in \eqref{eq:general-counterfactual} show that every feature norm is at most one.

Finally, for coverage, average $\phi((L,\xi), E_s)\phi((L,\xi), E_s)^\top$ over $(L,\xi)$:
\begin{equation*}
  \Sigma_E := \E_{(L,\xi)}\bigl[\phi((L,\xi), E_s)\phi((L,\xi), E_s)^\top\bigr]
  =\frac1{1+\rho^2} \begin{pmatrix}
      1 & 0 \\ 0 & \frac{\rho^2}{n}I_n
  \end{pmatrix}.
\end{equation*}
Uniform action sampling selects an informative action with probability $p$. The high-reward actions have zero factual feature, so
\begin{equation}
  \Sigma=\E\bigl[\phi(X,A)\phi(X,A)^\top\bigr]=
  p\Sigma_E
  +\frac1K\phi(S_0)\phi(S_0)^\top,
  \quad
  \sigma_0^2=\frac{p\rho^2}{n(1+\rho^2)}>0.
  \label{eq:general-covariance}
\end{equation}
The $S_0$ term is positive-semidefinite and affects only the first coordinate.  The first diagonal entry of $p\Sigma_E$ is no smaller than the other entries, so $\sigma_0^2$ is the exact minimum eigenvalue.
\end{proof}

Recall from \eqref{eq:D} that
\[
D=\frac{n(1+\rho^2)}{\rho^2v_B}.
\]
We identify $D$ as the scale of $V_\star$ in this construction.

\begin{lemma}
\label{lem:general-fisher}
For every true parameter $j\in[M]$,
\[
  \frac{e^{-1/4}}{20}D
  \le V_{\star,j}\le\frac{2D}{3}.
\]
Moreover, for the global curvature $\kappa$ of the common family,
\[
  \frac{3e^{-1/4}}8\frac{\kappa}{\sigma_0^2}
  \le D\le\frac{\kappa}{\sigma_0^2}.
\]
\end{lemma}

\begin{proof}
Let $I_E(\theta)$ be the Fisher information from one $E_s$, averaged over $(L,\xi)$. For every $\theta \in \Theta$, the logit of $E_s$ lies within $1/4$ of $-B$, as in the proof of Lemma~\ref{lem:binary-fisher}. Therefore, by \eqref{eq:logistic-comparison},
\begin{equation}
  e^{-1/4}v_B\Sigma_E
  \preceq I_E(\theta)
  \preceq e^{1/4}v_B\Sigma_E.
  \label{eq:general-IE}
\end{equation}
Since every $H_i$ has zero factual feature, we also have
\begin{equation*}
  I(\theta)=pI_E(\theta)
  +\frac1K\dot\mu(-b_\alpha)\phi(S_0)\phi(S_0)^\top.
\end{equation*}

Since $\phi(H_j)=0$, we have $q_r(H_j;\theta_j)=0$.  Moreover, $\phi^G(H_j)^\top\theta_j=0$, so $\dot\mu(\phi^G(H_j)^\top\theta_j)=\dot\mu(0)=1/4$. Consequently,
\[
q_\Delta(H_j;\theta_j)
=
-\frac14\phi^G(H_j),
\]
whose $(e_2,\ldots, e_{n+1})$-coordinate block is
\[
q'=-\frac{c_\alpha}{4}u_j.
\]
The safe term has no $(e_2,\ldots, e_{n+1})$-coordinates.  The upper side of \eqref{eq:general-IE} implies, for every $w\in\R^n$,
\[
  (0,w)^\top I(\theta_j)(0,w)
  \le\frac{pe^{1/4}v_B\rho^2}{n(1+\rho^2)}\norm{w}_2^2.
\]
Restricting \eqref{eq:variational} to vectors $(0,w)$ yields
\begin{align*}
  q_\Delta(H_j;\theta_j)^\top I_j^{-1}q_\Delta(H_j;\theta_j)
  &\ge\frac{n(1+\rho^2)}{pe^{1/4}v_B\rho^2}
       \norm{q'}_2^2\nonumber\\
  &=\frac{e^{-1/4}c_\alpha^2D}{16p}
  \ge\frac{e^{-1/4}D}{20},
\end{align*}
where we used \eqref{eq:calpha-lower} and $p\le1$.

For $V_{\star, j} \le 2D/3$, the lower side of \eqref{eq:general-IE} implies
\begin{equation*}
  I(\theta)^{-1}
  \preceq\frac{e^{1/4}}{pv_B}\Sigma_E^{-1}.
\end{equation*}
Again we check every nonzero term in the definition of $V_\star$.
\begin{itemize}
\item $H_i$ has zero $q_r$. Direct calculation gives
\[
  \phi^G(H_i)^\top\Sigma_E^{-1}\phi^G(H_i)
  =c_\alpha^2\left(
    \frac{\alpha^2}{B^2}+\frac{n(1+\rho^2)}{\rho^2}
  \right)
  \le2v_BD.
\]
The last inequality uses $\alpha^2/B^2\le1/16$ and $v_BD=n(1+\rho^2)/\rho^2\ge2n$. Since the $\dot \mu$ coefficient in $q_\Delta$ is at most $1/4$,
\[
  \norm[I(\theta)^{-1}]{q_\Delta(H_i;\theta)}^2
  \le\frac{e^{1/4}D}{8p}<\frac{2D}{3}.
\]
\item $S_0$ has zero effect and $\phi(S_0)^\top\Sigma_E^{-1}\phi(S_0)=b_\alpha^2/B^2\le1/16$. Thus
\[
  \norm[I(\theta)^{-1}]{q_r(S_0;\theta)}^2
  \le\frac{e^{1/4}}{256pv_B}<\frac{2D}{3}.
\]
\item $E_s$ has zero effect, $\phi(x,E_s)^\top\Sigma_E^{-1}\phi(x,E_s)=d$, and derivative at most $e^{1/4}v_B$.  Hence
\[
  \norm[I(\theta)^{-1}]{q_r(x,E_s;\theta)}^2
  \le\frac{e^{3/4}v_Bd}{p}<\frac{2D}{3}.
\]
\end{itemize}
These comparisons follow from $p\ge3/8$, $v_B<1/5$, $d/n\le46/45$, and $\rho^2/(1+\rho^2)\le1/2$. This proves $\frac{e^{-1/4}}{20}D\le V_{\star,j}\le\frac{2D}{3}$.

To bound the curvature, use \eqref{eq:general-covariance} to write
\begin{equation*}
  D=\frac{p}{v_B\sigma_0^2}.
\end{equation*}
Over $\Theta$, every factual and counterfactual logit has absolute value at most $B+1/4$. For an action $E_s$, one of the two context signs gives a logit of magnitude at least $B$. Hence $v_B^{-1}\le\kappa\le e^{1/4}v_B^{-1}$.  Combining this with $p\in[3/8,1]$ gives
\[
  \frac{D}{\kappa/\sigma_0^2}
  =\frac{p}{\kappa v_B}
  \in\left[\frac{3e^{-1/4}}8,1\right],
\]
equivalently, $\frac{3e^{-1/4}}8\frac{\kappa}{\sigma_0^2}\le D\le\frac{\kappa}{\sigma_0^2}$.
\end{proof}

\subsection{Main Proof of Theorem~\ref{thm:general}}

We first bound the KL divergence between observations in different environments.

\begin{lemma} For every context $x=(L,\xi)$, every pair
$j,k\in[M]$, and every informative replica $E_s$,
\begin{equation*}
  \operatorname{KL}\!\left(
    P_j(\,\cdot\mid x,E_s)
    \,\middle\|\,
    P_k(\,\cdot\mid x,E_s)
  \right)
  \le\frac{2e^{1/4}\alpha^2}{D}.
\end{equation*}
The factual reward distributions of all non-informative actions are identical in all environments.
\label{lem:general-kl}
\end{lemma}

\begin{proof} 

Conditionally on $x=(L,\xi)$, the logits of $E_s$ in environments $j$ and $k$ differ by
\[
  \frac{\xi\rho\alpha}{\sqrt{1+\rho^2}}
  \bigl(u_k(L)-u_j(L)\bigr).
\]
Its squared magnitude is at most $4\rho^2\alpha^2/[n(1+\rho^2)]$.  The interval between the logits remains between $-B-1/4$ and $-B+1/4$. Applying Lemma~\ref{lem:one-step-kl} gives
\[
  \operatorname{KL}(P_j\|P_k)
  \le\frac12e^{1/4}v_B
       \frac{4\rho^2\alpha^2}{n(1+\rho^2)}
  =\frac{2e^{1/4}\alpha^2}{D}.
\]
The remaining part of the claim follows from Lemma~\ref{Lem:general-costs}.
\end{proof}

Define the constants used in the proof by
\[
  a_{\rm gen}=\frac{\Delta_ED\log M}{8e^{1/4}},
  \quad
  c_{\rm gen}=\frac{v_0}{32},
  \quad
  \widetilde c_2=\left(\frac{\Delta_E}{8e^{1/4}}\right)^{1/3}
      \left(\frac{v_0}{32}\right)^{2/3},
  \quad
  C_2=\frac{256\Delta_E}{e^{1/4}v_0}.
\]

\begin{proof}[Proof of Theorem~\ref{thm:general}]
For a given $T$, choose
\begin{equation*}
  \alpha=\left(\frac{a_{\rm gen}}{c_{\rm gen}T}\right)^{1/3}.
\end{equation*}
By \eqref{eq:general-size}, $\log M\le\min\{\log K,d\}$. The horizon condition in Theorem~\ref{thm:general} therefore gives $T\ge C_2D\log M$, so $\alpha^3\le1/64$ and $\alpha\le1/4$.  Thus the construction is valid. Fix an arbitrary adaptive, possibly randomized policy $\pi$.  Let $N_E(T)$ count the selections of informative actions, and define their average over environments by
\[
  \overline N_E(T)
  =\frac1{M}
   \sum_{j=1}^{M}\E_jN_E(T).
\]

We consider two cases according to the expected number of informative actions:
\begin{itemize}
\item \textbf{Case 1.} If
    \[
    \overline N_E(T)
    \ge\frac{D\log M}{8e^{1/4}\alpha^2},
    \]
since every informative selection costs at least $\Delta_E$, we get
    \[
    \max_j\E_j[R_T+V_T]
    \ge\frac{\Delta_ED\log M}         {8e^{1/4}\alpha^2}
    =\frac{a_{\rm gen}}{\alpha^2}.
    \]
\item \textbf{Case 2.} Now suppose
    \[
    \overline N_E(T)
    <\frac{D\log M}{8e^{1/4}\alpha^2}.
    \]
Let the true environment index $Z$ be uniform on $[M]$. Lemmas~\ref{lem:adaptive-kl} and  \ref{lem:general-kl} give
    \begin{align}
    I(Z;\cH_t)
    &\le\frac1{M^2}
       \sum_{j,k=1}^{M}
       \operatorname{KL}(P_j^t\|P_k^t)\nonumber\\
    &\le\frac{2e^{1/4}\alpha^2}{D}
       \frac1{M}
       \sum_{j=1}^{M}\E_jN_E(t-1)\nonumber\\
    &\le\frac{2e^{1/4}\alpha^2}{D}\overline N_E(T)
    <\frac14\log M.
    \label{eq:general-mutual-info}
    \end{align}
Infer an environment index from the action by setting
    \[
    \widehat Z_t=i\quad\text{if }A_t=H_i,
    \quad
    \widehat Z_t=1\;\text{if }A_t\text{ is not a high-reward action}.
    \]
By the data-processing statement in Lemma~\ref{lem:info-tools}, $I(Z;\widehat Z_t)\le I(Z;\cH_t)$. Since $M\ge4$, Fano's inequality and \eqref{eq:general-mutual-info} imply
    \[
    \Pp(\widehat Z_t\ne Z)
    \ge1-\frac{(\log M)/4+\log2}{\log M}
    \ge\frac14.
    \]
The event $\{\widehat Z_t\ne Z\}$ is contained in $\{A_t\ne H_Z\}$.  Hence
\[
  \frac1{M}
  \sum_{j=1}^{M}\Pp_j(A_t\ne H_j)
  \ge\frac14.
\]
Every action in this event has cost at least $v_0\alpha/8$. Averaging over environments and summing over $t$ gives the following bound, since the maximum is at least the average:
\[
  \max_j\E_j[R_T+V_T]
  \ge\frac{v_0\alpha T}{32}
  =c_{\rm gen}\alpha T.
\]
\end{itemize}

The choice of $\alpha$ balances the two cases:
\[
  \frac{a_{\rm gen}}{\alpha^2}
  =c_{\rm gen}\alpha T
  =a_{\rm gen}^{1/3}c_{\rm gen}^{2/3}T^{2/3}
  =\widetilde c_2(D\log M)^{1/3}T^{2/3}.
\]
Thus every policy has worst-case expected joint loss at least $\widetilde c_2(D\log M)^{1/3}T^{2/3}$.

It remains to express the bound in terms of $K$ and $d$. If $M=\lfloor K/2\rfloor$, then $M\ge K/3$ because $K\ge8$.  Otherwise, $n\ge45$ implies
\[
  M=\lfloor e^{n/32}\rfloor\ge e^{n/64}.
\]
Therefore
\[
  \log M
  \ge\min\{\log(K/3),n/64\}
  \ge\frac1{128}\min\{\log K,d\},
\]
where the last step uses $n=d-1\ge d/2$ and $\log(K/3)\ge(\log K)/3$ for $K\ge8$. Taking $c_2=128^{-1/3}\widetilde c_2$ proves \eqref{eq:general-lower} for the family indexed by $\mathcal J=[M]$. The comparison $V_{\star,j}\asymp D$ follows from Lemma~\ref{lem:general-fisher}.
\end{proof}

\section{Proof of Theorems~\ref{thm:fixed} and \ref{thm:adaptive}}

We first establish Lemma~\ref{lem:screening}.

\begin{proof}[Proof of Lemma~\ref{lem:screening}]
By definition, the true effect of $a^\star(x)$ belongs to $[-\tau,\tau]$, and also belongs to its effect interval by assumed validity. Thus $a^\star(x)$ is retained. Also by validity and reward optimism,
\[
r_{\theta^\star}(x,a^\star(x))
\le\overline r(x,a^\star(x))
\le\overline r(x,A),
\qquad
r_{\theta^\star}(x,A)\ge\underline r(x,A),
\]
which proves the regret bound.

For the violation bound, write the selected effect interval as $[l,u]$ and set $\Delta_A=\Delta_{\theta^\star}(x,A)$. The interval contains $\Delta_A$ and intersects $[-\tau,\tau]$. If $\Delta_A>\tau$, then $u\ge\Delta_A$ and $l\le\tau$, hence $\Delta_A-\tau\le u-l$. If $\Delta_A<-\tau$, then $l\le\Delta_A$ and $u\ge-\tau$, hence $-\Delta_A-\tau\le u-l$. If $|\Delta_A|\le\tau$, the LHS is zero.
\end{proof}

We first prove both theorems using the valid intervals and width bound in Lemma~\ref{lem:master-interface}, then prove that lemma. We work on the probability-one event that $\cA_\tau(X_t)\ne \varnothing$ for all $t\le T$ unless otherwise stated.

\begin{proof}[Proof of Theorem~\ref{thm:fixed} assuming Lemma~\ref{lem:master-interface}]
Work on the event that the confidence intervals are valid. Each of the first $m$ rounds contributes at most one to each of regret and violation. Lemma~\ref{lem:screening} bounds every later contribution by the interval width. Together with Lemma~\ref{lem:master-interface}, this gives
\[
\max\{R_T,V_T\}
\le
m+1024(T-m)\sqrt{\frac{V_\star\ell}{m}}
\le
m+1024T\sqrt{\frac{V_\star\ell}{m}}
\]
as desired.

For the specific choice of $m$, put $a=1024^2V_\star\ell T^2$.  If that choice of $m$ has $m<T$, then $m\le n_0+a^{1/3}+1$ and $m\ge a^{1/3}$.  The second term in \eqref{eq:fixed-bound} equals $\sqrt{a/m}\le a^{1/3}$, so
\[
\max\{R_T,V_T\}
\le n_0+1+2a^{1/3}.
\]
If $m=T$, then $T\le n_0+a^{1/3}+1$, while $R_T,V_T\leq T$ implies the result.
\end{proof}

The following lemma bounds the sample size at which further exploration stops.

\begin{lemma}
\label{lem:no-trigger}
Define $n_{\mathrm{crit}}=\max\left\{n_0,(1024^2V_\star\ell T^2)^{1/3} \right\}$. On the event that Lemma~\ref{lem:master-interface} guarantees, a checkpoint with $n_j\ge n_{\mathrm{crit}}$ cannot trigger subsequent exploration.
\end{lemma}

\begin{proof}
The definition of $n_{\mathrm{crit}}$ gives $1024\sqrt{V_\star\ell/n_j}\le n_j/T$.  Combine this with \eqref{eq:master-width}.
\end{proof}

\begin{proof}[Proof of Theorem~\ref{thm:adaptive} assuming Lemma~\ref{lem:master-interface}]
If $n_0\ge T$, the algorithm explores uniformly throughout, and the result follows from $R_T,V_T\le T$.  Suppose $n_0<T$ and work on the probability $1-\delta$ event of Lemma~\ref{lem:master-interface} and the probability-one event that $\cA_\tau(X_t)\ne \varnothing$ for all $t\le T$.

If the width test never fails, set $\bar n=n_0$. Otherwise, let $j$ be the active checkpoint at the last failed test and set $\bar n=n_{j+1}$, the target sample size of the ensuing batch, even if the horizon truncates that batch. By Lemma~\ref{lem:no-trigger}, every failed test occurs at $n_j<n_{\mathrm{crit}}$, and $n_{j+1}\le2n_j$. Hence
\begin{equation*}
\bar n\le2n_{\mathrm{crit}}.
\end{equation*}
Partition the rounds into three types.

\begin{itemize}
\item \textbf{Exploration rounds.} There are at most $\bar n$ exploration rounds. Their total contribution to each of regret and violation is at most $\bar n$.
\item \textbf{Trigger rounds.} Each checkpoint generates at most one discarded trigger round. There are at most $J+1$ checkpoints, so their contribution is at most $J+1$.
\item \textbf{Exploitation rounds.} Let $\mathcal T_j$ be the set of exploitation rounds on which checkpoint $j$ is active.  Lemma~\ref{lem:screening} and the exploitation condition $W_j(X_t)\le n_j/T$ in Algorithm~\ref{alg:adaptive-ete} imply \[\sum_{t\in\mathcal T_j}\bigl[r_{\theta^\star}(X_t,a^\star(X_t))-r_{\theta^\star}(X_t,A_t)\bigr]\le|\mathcal T_j|\frac{n_j}{T}, \] and the same bound holds for the violations on $\mathcal T_j$. Every active checkpoint satisfies $n_j\le\bar n$, so
    \[
    \sum_j|\mathcal T_j|\frac{n_j}{T}
    \le\frac{\bar n}{T}\sum_j|\mathcal T_j|
    \le\bar n.
    \]
Thus exploitation contributes at most $\bar n$, including when there are no exploitation rounds.
\end{itemize}
Adding the three contributions and using $\bar n\le2n_{\mathrm{crit}}$ gives
\[
\max\{R_T,V_T\}
\le2\bar n+J+1
\le4n_{\mathrm{crit}}+J+1.
\]
Taking the minimum with $R_T,V_T\leq T$ proves the bound on $\max \{R_T, V_T\}$.

Finally, \eqref{eq:gamma-corrected} implies $n_0=\wt O(\kappa d/\sigma_0^2)$, while the definition of $\ell$ gives $\ell=O(\log(T/\delta)+\min\{\log K,d\})$. For more than one action, $\min\{\log K,d\}$ is bounded below by a positive universal constant, so substituting into $n_{\mathrm{crit}}$ and absorbing logarithmic factors gives the stated rate. With a single action, Assumption~\ref{ass:fair} gives $R_T=V_T=0$, so the rate holds directly.
\end{proof}

\subsection{Proof of Lemma~\ref{lem:master-interface}}

We index the adaptive algorithm's fits by $j$, with sample sizes $n_j$ as in Section~\ref{sec:adaptive-ete}. For fixed-length ETE, the same argument applies to the single fit of size $m$, using $m$ in place of $n_j$.

To prove Lemma~\ref{lem:master-interface}, we specify $\Gamma_{\rm corr}$ and establish concentration and existence of the MLE at every checkpoint. Recall that we set $\delta\in(0,1)$ and defined
\[
J=\lceil\log_2T\rceil,
\quad
M_{\rm dir}=\min\{KT,5^d\},
\quad
\ell=
\log\frac{32(J+1)(T+2)M_{\rm dir}}{\delta}.
\]
Now define
\begin{equation}
\Gamma_{\mathrm{corr}}
=
\max\left\{
\begin{aligned}
&128\left[d\log 6+
\log\frac{16(J+1)(T+2)M_{\rm dir}}{\delta}\right],\\
&533\ell,\\
&2\log\frac{8d(J+1)}{\delta},\\
&2d\log(4eT)+2\log\frac{16(J+1)}{\delta}
\end{aligned}
\right\}
\label{eq:gamma-corrected}
\end{equation}

and $n_0 = \lceil \kappa \Gamma_{\rm corr} /\sigma_0^2 \rceil$.

The following result from \cite{jun2021improved} is stated with larger numerical constants.

\begin{lemma}[\citet{jun2021improved}]
\label{lem:jun}
Let $x_1,\ldots,x_n\in\R^d$ satisfy $\norm{x_s}_2\le1$.  Let the responses be independent with $R_s\sim\operatorname{Bernoulli}(\mu(x_s^\top\theta^\star))$.  Define
\[
H(\theta)=\sum_{s=1}^n
\dot\mu(x_s^\top\theta)x_sx_s^\top
\]
and assume $H(\theta^\star)\succ0$.  Let $\wh\theta$ denote the logistic MLE whenever it exists as a finite vector, with $\wh H=H(\wh\theta)$. Let $n_{\mathrm{eff}}$ be the number of distinct design vectors, fix $z\in\R^d$ with $\norm{z}_2\le1$, and let $\varepsilon\le e^{-1}$.  If
\begin{equation}
\max_{s\le n}\norm[H(\theta^\star)^{-1}]{x_s}^2
\le
\frac{1}{64\left[
d\log 6+\log((2+n_{\mathrm{eff}})/\varepsilon)
\right]},
\label{eq:jun-leverage}
\end{equation}
then, conditional on the fixed design, there is a bad event $B_\varepsilon$ of probability at most $\varepsilon$ such that, on $\{\wh\theta\text{ is finite}\}\setminus B_\varepsilon$,
\begin{align}
|z^\top(\wh\theta-\theta^\star)|
&\le
3.5\sqrt{\log\frac{2(2+n_{\mathrm{eff}})}{\varepsilon}}
\,\norm[H(\theta^\star)^{-1}]{z},
\label{eq:jun-true}\\
\frac1{\sqrt{2.2}}\norm[H(\theta^\star)^{-1}]{u}
&\le
\norm[\wh H^{-1}]{u}
\le
\sqrt{2.2}\norm[H(\theta^\star)^{-1}]{u}
\quad\text{for every }u\in\R^d,
\label{eq:jun-fisher-comparison}
\end{align}
and, in particular,
\begin{equation}
|z^\top(\wh\theta-\theta^\star)|
\le
5.2\sqrt{\log\frac{2(2+n_{\mathrm{eff}})}{\varepsilon}}
\,\norm[\wh H^{-1}]{z}.
\label{eq:jun-empirical}
\end{equation}
\end{lemma}

Although the adaptive ETE requests additional exploration in response to its current confidence widths, the observations kept for fitting are not selected using their own contexts. Each batch is scheduled before those contexts are observed, and the triggering round is discarded. They can therefore be viewed as consecutive draws from a single sequence of iid uniform-explorations. The following lemma formalizes this observation.

\begin{lemma}
\label{lem:indep}
Let $Q$ denote the law of one uniform factual observation $Z=(X,A,R)$, where $X\sim P_X$, $A\sim\operatorname{Unif}(\mathcal A)$ independently of $X$, and
\[
 R\mid X,A\sim
 \operatorname{Bernoulli}\!\left(r_{\theta^\star}(X,A)\right).
\]
The adaptive ETE can be coupled with an infinite iid sequence $(Z_s^E)_{s\ge1}$ drawn from $Q$ such that, whenever checkpoint $j$ is reached, the observations used for fitting are exactly $Z_1^E,\ldots,Z_{n_j}^E$.
\end{lemma}

\begin{proof}
Generate an infinite sequence of uniform exploration observations \emph{in advance} and reveal them only when requested. The initial exploration phase reveals its first $n_0$ entries. If the width test fails at round $t$, the algorithm schedules a batch of future rounds for uniform exploration \emph{before} observing any of their contexts. Conditional on everything observed through round $t$, these future contexts are fresh iid draws from $P_X$, their actions are independent uniform draws, and their rewards follow the Bernoulli law. The new batch can therefore be identified with the next unused entries of this sequence.
\end{proof}

At every checkpoint $j$, define the Fisher matrix at the true parameter
\begin{equation}
H_j^\star
=
\sum_{s=1}^{n_j}
\dot\mu((\phi_s^E)^\top\theta^\star)
\phi_s^E(\phi_s^E)^\top.
\label{eq:true-empirical-fisher}
\end{equation}

Here $\phi^E_s$ denotes the feature vector of the $s$-th exploration observation used to fit the MLE.

\begin{lemma}
\label{lem:fisher-concentration}
With failure probability at most $\delta/8$, simultaneously at every checkpoint,
\begin{equation}
H_j^\star\succeq\frac{n_j}{2}I_\star.
\label{eq:fisher-lower}
\end{equation}
On this event,
\begin{equation}
(H_j^\star)^{-1}
\preceq
\frac2{n_j}I_\star^{-1}
\preceq
\frac{2\kappa}{n_j\sigma_0^2}I_d
\label{eq:fisher-inverse}
\end{equation}
and
\begin{equation}
\max_{s\le n_j}
\norm[(H_j^\star)^{-1}]{\phi_s^E}^2 \le \frac{2\kappa}{n_j\sigma_0^2} \le \frac2{\Gamma_{\mathrm{corr}}}.
\label{eq:sample-leverage}
\end{equation}
\end{lemma}

Note that the last quantity satisfies the requirement \eqref{eq:jun-leverage} in Lemma~\ref{lem:jun}.

\begin{proof}
By Lemma~\ref{lem:indep}, the summands of \eqref{eq:true-empirical-fisher} are iid with expectation $I_\star$.  Set
\[
Z_s = I_\star^{-1/2}\dot\mu((\phi_s^E)^\top\theta^\star)\phi_s^E(\phi_s^E)^\top I_\star^{-1/2}.
\]
Then $\E Z_s=I_d$ (the $d\times d$ identity matrix). Since $Z_s$ has rank at most one, the bounds $I_\star^{-1}\preceq(\kappa/\sigma_0^2)I_d$, $\norm[2]{\phi} \le 1$, and $\dot\mu\le1/4$ give
\[
\lambda_{\max}(Z_s)
=
\dot\mu((\phi_s^E)^\top\theta^\star)
\norm[I_\star^{-1}]{\phi_s^E}^2
\le
\frac{\kappa}{4\sigma_0^2}.
\]
Matrix Chernoff with relative deviation $1/2$ now gives
\begin{equation*}
\Pp\left(
H_j^\star\nsucceq\frac{n_j}{2}I_\star
\right)
\le
d\exp\left(-\frac{n_j\sigma_0^2}{2\kappa}\right).
\end{equation*}
Since $n_j\ge n_0\ge\kappa\Gamma_{\mathrm{corr}}/\sigma_0^2$, a union bound over at most $J+1$ checkpoints is bounded by
\[
d(J+1)e^{-\Gamma_{\mathrm{corr}}/2}
\le\frac\delta8,
\]
using the third line of \eqref{eq:gamma-corrected}.  This proves \eqref{eq:fisher-lower}.  Taking inverses and using $I_\star^{-1}\preceq(\kappa/\sigma_0^2)I_d$ give \eqref{eq:fisher-inverse}; $\norm[2]{\phi} \le 1$ and $n_j\ge n_0$ give \eqref{eq:sample-leverage}.
\end{proof}

Every later application of Lemma~\ref{lem:jun} has $n_{\mathrm{eff}}\le T$ and
\[
\varepsilon\ge\frac{\delta}{16(J+1)M_{\rm dir}}.
\]
The first line of \eqref{eq:gamma-corrected} therefore gives
\[
\frac2{\Gamma_{\mathrm{corr}}}
\le
\frac1{64\left[
d\log 6+
\log\frac{16(J+1)(T+2)M_{\rm dir}}{\delta}
\right]}
\le
\frac1{64\left[
d\log 6+\log((2+n_{\mathrm{eff}})/\varepsilon)
\right]},
\]
which is \eqref{eq:jun-leverage}. Moreover, $T\ge2$ implies $J+1\ge2$, and every failure probability used below is at most $\delta/[16(J+1)]\le1/32<e^{-1}$.  Thus the failure probabilities satisfy the condition in Lemma~\ref{lem:jun}.

The next lemma shows that the MLE and estimated Fisher inverse are well defined at every checkpoint with high probability.

\begin{lemma}
\label{lem:finite-mle}
Apart from an event of probability at most $\delta/8$, every checkpoint satisfying \eqref{eq:fisher-lower} has a unique finite MLE, and its estimated Fisher matrix is positive definite.
\end{lemma}

\begin{proof}
Fix a checkpoint and write its features as $x_s=\phi_s^E$.  Write $p_s=\mu(x_s^\top\theta^\star)$.  The negative log-likelihood is
\[
L(\theta)
=
\sum_s
\left[
\log(1+e^{x_s^\top\theta})-R_sx_s^\top\theta
\right].
\]
Define its recession function in direction $v$ by
\begin{equation}
L_\infty(v)
:=
\lim_{c\to \infty} \frac{L(\theta+cv) - L(\theta)}{c}
=
\sum_{s:R_s=0}(x_s^\top v)_+
+
\sum_{s:R_s=1}(-x_s^\top v)_+.
\label{eq:recession}
\end{equation}

where $(a)_+ := \max (a,0)$. The limit in \eqref{eq:recession} does not depend on $\theta$.

The lower bound in \eqref{eq:fisher-lower} implies that $\{x_1, x_2, \ldots, x_{n_j}\}$ spans $\R^d$.  If some nonzero $v$ has $L_\infty(v)=0$, then
\begin{itemize}
\item $x_s^\top v \le 0$ for $R_s = 0$;
\item $x_s^\top v \ge 0$ for $R_s= 1$.
\end{itemize}

Equivalently,

\begin{itemize}
\item $x_s^\top v < 0$ implies $R_s=0$;
\item $x_s^\top v > 0$ implies $R_s =1$.
\end{itemize}

Hence the hyperplane $x^\top v=0$ separates the observations with labels $R_s = 0$ and $R_s = 1$. At least one projection is nonzero because $\{x_1, x_2, \ldots, x_{n_j}\}$ spans $\R^d$. For each nonzero projection, the corresponding term in the loss strictly decreases as $c$ increases. Thus moving in direction $v$ decreases the loss, so no finite $\theta$ attains its minimum.

Conversely, if $L_\infty(v)>0$ for every unit vector $v$, continuity of $L_\infty$ and compactness of the unit sphere give a positive lower bound for $L_\infty$. Then $L \to \infty$ as $\norm[2]{\theta} \to \infty$, thus $L$ attains a finite minimum inside some sufficiently large compact ball.

We have established that the following are equivalent:
\begin{itemize}
\item the MLE is not finite;
\item there exists a direction with $L_\infty (v) = 0$;
\item there exists a hyperplane that separates the two labels.
\end{itemize}

Fix a unit direction $v$ and let $S(v)=\{s:x_s^\top v\ne0\}$.  For $s\in S(v)$, let $r_s=p_s$ when separation requires label one (i.e., $x_s^\top v>0$) and $r_s=1-p_s$ when it requires label zero (i.e., $x_s^\top v<0$). Then
\begin{equation*}
\Pp(\text{all required labels match}\mid x_{1:n_j})
=
\prod_{s\in S(v)}r_s.
\end{equation*}
Using $-\log r_s\ge1-r_s\ge p_s(1-p_s)$ gives
\begin{equation}
\Pp(\text{all required labels match}\mid x_{1:n_j})
\le
\exp\left(-\sum_{s\in S(v)}p_s(1-p_s)\right).
\label{eq:face-probability}
\end{equation}
Since $|x_s^\top v|\le1$, on \eqref{eq:fisher-lower},
\begin{align*}
\sum_{s\in S(v)}p_s(1-p_s)
&\ge
\sum_s p_s(1-p_s)(x_s^\top v)^2\\
&=v^\top H_j^\star v
\ge\frac{n_j}{2}v^\top I_\star v
\ge\frac{n_j\sigma_0^2}{2\kappa}.
\end{align*}

Hence
\[
\Pp(\text{one particular separating pattern})
\le
\exp\left(- \frac{n_j\sigma_0^2}{2\kappa}\right).
\]

To take a union bound over separating directions, we count the possible sign patterns. The required labels depend only on
\[
\sign(x_1^\top v) , \ldots, \sign (x_{n_j}^\top v)
\]
and these signs are constant on each face of the arrangement of hyperplanes $x_s^\top v = 0$. The zero-dimensional face contributes one.  For $m\ge1$, a face whose span has dimension $m$ is contained in the intersection of $d-m$ linearly independent homogeneous hyperplanes; there are at most $n_j^{d-m}$ choices for those hyperplanes. Inside an $m$-dimensional intersection, we have at most $2\sum_{k=0}^{m-1}\binom{n_j-1}{k}\le2(d+1)n_j^m$ chambers.  Adding the zero-dimensional face and summing over $m=1,\ldots,d$ yields the bound
\[
N_{\mathrm{face}}
\le2(d+1)^2n_j^d
\le(4en_j)^d
\]
which includes lower-dimensional faces.  Combining the two orientations, \eqref{eq:face-probability}, and $n_j\le T$ gives
\[
\Pp(\text{no finite MLE and \eqref{eq:fisher-lower}}\mid x_{1:n_j})
\le
2(4eT)^d
\exp\left(-\frac{n_j\sigma_0^2}{2\kappa}\right).
\]
The fourth line of \eqref{eq:gamma-corrected} gives
\[
2(J+1)(4eT)^d e^{-\Gamma_{\mathrm{corr}}/2}
\le\frac\delta8.
\]
Now a union bound over checkpoints proves the finite-MLE part of the claim.

Finally, the Hessian at every finite parameter is $$\sum_s\dot\mu(x_s^\top\theta)x_sx_s^\top$$ which is positive definite because $\dot\mu(x_s^\top\theta)>0$ and the design vectors span $\R^d$. Hence $L$ is strictly convex, the finite minimizer is unique, and $\wh H_j$ is invertible.
\end{proof}

The next lemma combines these events with logit confidence bounds that hold simultaneously at the contexts where each checkpoint is used. For adaptive ETE, let $\mathcal C_{j,t}$ be the event that, before observing $X_t$, checkpoint $j$ has been fitted and the algorithm is about to compute $W_j(X_t)$ in the while loop of Algorithm~\ref{alg:adaptive-ete}. This event is determined by the history before $X_t$. It includes both exploitation and trigger rounds, and excludes scheduled exploration rounds. For fixed-length ETE, use its sole checkpoint of size $m$ and set $\mathcal C_{0,t}=\{t>m\}$.

\begin{lemma}
\label{lem:simultaneous-logit}
With probability at least $1-\delta$, every checkpoint MLE that the algorithm calculates is finite and unique, its estimated Fisher matrix is invertible, and simultaneously for every such checkpoint $j$,
\begin{equation}
\frac1{\sqrt{2.2}}\norm[(H_j^\star)^{-1}]{u}
\le
\norm[\wh H_j^{-1}]{u}
\le
\sqrt{2.2}\norm[(H_j^\star)^{-1}]{u}
\quad\text{for all }u\in\R^d,
\label{eq:simultaneous-fisher-comparison}
\end{equation}
and, at every round with $\mathcal C_{j,t}$, for all $a\in\cA$ and $z\in\{\phi(X_t,a),\phi^G(X_t,a)\}$,
\begin{equation}
|z^\top(\wh\theta_j-\theta^\star)|
\le
11\sqrt\ell\,\norm[\wh H_j^{-1}]{z}
=b_j(z).
\label{eq:simultaneous-logit}
\end{equation}
\end{lemma}

\begin{proof}
Work on the events of Lemmas~\ref{lem:fisher-concentration} and~\ref{lem:finite-mle}.  Their failure probabilities are at most $\delta/8$ each, and \eqref{eq:sample-leverage} ensures Lemma~\ref{lem:jun} is applicable. Let $\mathcal E_{j,\mathrm{base}}$ denote the event that checkpoint $j$ satisfies the Fisher lower bound \eqref{eq:fisher-lower} and has a finite MLE.

At each checkpoint, apply Lemma~\ref{lem:jun} with any fixed unit direction and failure probability $\delta/[16(J+1)]$.  The global comparison \eqref{eq:jun-fisher-comparison} belongs to the same event.  A union bound over checkpoints costs at most $\delta/16$ and proves \eqref{eq:simultaneous-fisher-comparison}.

We bound the logit errors using one of two arguments, according to which gives the smaller logarithmic factor.

\begin{itemize}
\item \textbf{Action-wise route:} when $KT\le5^d$. Fix $(j,t,a)$ and one of the two deterministic maps $z(x)=\phi(x,a)$ or $z(x)=\phi^G(x,a)$.  For every deterministic $x$, apply \eqref{eq:jun-empirical} with
    \[
    \varepsilon
    =
    \frac{\delta}{16(J+1)KT}.
    \]
Because $n_{\mathrm{eff}}\le T$,
    \[
    \log\frac{2(2+n_{\mathrm{eff}})}{\varepsilon}
    \le\log\frac{32(J+1)KT(T+2)}{\delta}=\ell.
    \]
Let $F(x)$ be the event of the failure of the resulting bound. The bad-event formulation of Lemma~\ref{lem:jun}, first conditional on the checkpoint design and then averaged over it, gives $\Pp(\mathcal E_{j,\mathrm{base}}\cap F(x))\le\varepsilon$ for every fixed $x$.  Since, on $\mathcal C_{j,t}$, the base event and estimator are determined before $X_t$ arrives, the tower property gives
    \begin{align*}
    &\Pp(\mathcal E_{j,\mathrm{base}}\cap\mathcal C_{j,t}\cap F(X_t))\\
    &=\E\left[
    \ind_{\mathcal E_{j,\mathrm{base}}\cap\mathcal C_{j,t}}
    \int\ind_{F(x)}\,P_X(dx)
    \right]\\
    &=\int
    \Pp(\mathcal E_{j,\mathrm{base}}\cap\mathcal C_{j,t}\cap F(x))\,P_X(dx)\\
    &\le
    \int \Pp(\mathcal E_{j,\mathrm{base}}\cap F(x))\,P_X(dx)
    \le\varepsilon.
    \end{align*}
There are at most $2(J+1)KT$ choices of direction, so the total failure probability is at most $\delta/8$, with coefficient $5.2<11$.
\item \textbf{Dimension-wise route:} when $5^d<KT$. For each checkpoint choose a $1/2$-net $\mathcal N$ of the Euclidean unit sphere with $|\mathcal N|\le5^d$.  Conditional on the design, for $h\in\mathcal N$ define
    \[
    z_{j,h}
    =
    \frac{(H_j^\star)^{1/2}h}
    {\norm{(H_j^\star)^{1/2}h}_2}.
    \]
This direction is allowed to depend on the design because Lemma~\ref{lem:jun} is applied after conditioning on that design.  Apply \eqref{eq:jun-true} with $\varepsilon=\delta/[16(J+1)5^d]$.  The same logarithmic calculation gives
    \[
    |h^\top(H_j^\star)^{1/2}(\wh\theta_j-\theta^\star)|
    \le3.5\sqrt\ell
    \quad(h\in\mathcal N).
    \]
The $1/2$-net inequality $\norm{v}_2\le2\max_{h\in\mathcal N}|h^\top v|$ therefore yields
    \begin{equation*}
    \norm[H_j^\star]{\wh\theta_j-\theta^\star}
    \le7\sqrt\ell.
    \end{equation*}
For every $z\in\R^d$, Cauchy--Schwarz and \eqref{eq:simultaneous-fisher-comparison} give
    \begin{align*}
    |z^\top(\wh\theta_j-\theta^\star)|
    &\le
    \norm[(H_j^\star)^{-1}]{z}
    \norm[H_j^\star]{\wh\theta_j-\theta^\star}\\
    &\le
    7\sqrt{2.2\ell}\,\norm[\wh H_j^{-1}]{z}
    <11\sqrt\ell\,\norm[\wh H_j^{-1}]{z}.
    \end{align*} 
This argument has total failure probability at most $\delta/16$.
\end{itemize}

The failure probabilities are bounded as follows:
\[
\begin{array}{c|c}
\text{event} & \text{failure probability}\\ \hline
\text{Fisher concentration} & \delta/8\\
\text{finite MLEs} & \delta/8\\
\text{empirical/true Fisher comparison} & \delta/16\\
\text{selected directional route} & \le\delta/8
\end{array}
\]
whose sum is $7\delta/16<\delta$.
\end{proof}

We can now prove Lemma~\ref{lem:master-interface}.

\begin{proof}[Proof of Lemma~\ref{lem:master-interface}]
Work on the event of Lemma~\ref{lem:simultaneous-logit}, where each MLE is finite and each $\wh H_j$ is invertible. By monotonicity of $\mu$, the bound \eqref{eq:simultaneous-logit} makes the factual and counterfactual reward intervals valid. Subtracting their endpoints gives a valid effect interval.

To prove the width bound \eqref{eq:master-width}, combine \eqref{eq:simultaneous-fisher-comparison} and \eqref{eq:fisher-inverse}. For every queried direction $z$, this gives
\begin{equation*}
b_j(z)
\le
11\sqrt{2.2\ell}\,\norm[(H_j^\star)^{-1}]{z}
\le
11\sqrt{\frac{4.4\ell}{n_j}}\,
\norm[I_\star^{-1}]{z}.
\end{equation*}
For $\norm{z}_2\le1$, using $I_\star^{-1}\preceq(\kappa/\sigma_0^2)I_d$ and $n_j\ge n_0$,
\begin{equation*}
b_j(z)^2
\le
\frac{121\cdot4.4\,\ell\kappa}{n_j\sigma_0^2}
\le
\frac{121\cdot4.4\,\ell}{\Gamma_{\mathrm{corr}}}
\le
\frac{532.4}{533}
<1.
\end{equation*}

Fix $z$ and write $\eta^\star=z^\top\theta^\star$, $\wh\eta=z^\top\wh\theta_j$, and $b=b_j(z)$.  On the confidence event, $|\wh\eta-\eta^\star|\le b\le1$.  Every point $u$ in $[\wh\eta-b,\wh\eta+b]$ therefore satisfies $|u-\eta^\star|\le2$.  Lemma~\ref{lem:logistic-geometry} gives $\dot\mu(u)\le e^2\dot\mu(\eta^\star)$ throughout that interval.  Hence
\begin{align}
\mu(\wh\eta+b)-\mu(\wh\eta-b)
&=
\int_{\wh\eta-b}^{\wh\eta+b}\dot\mu(u)\,du\notag\\
&\le2e^2\dot\mu(\eta^\star)b\notag\\
&\le
341\sqrt{\frac{\ell}{n_j}}
\norm[I_\star^{-1}]{\dot\mu(\eta^\star)z},
\label{eq:probability-width}
\end{align}
because
\[
2e^2\cdot11\sqrt{4.4}=340.987\ldots<341.
\]

For $z=\phi(x,a)$, the last vector in \eqref{eq:probability-width} is $q_r(x,a)$.  Thus
\begin{equation}
\overline r_j(x,a)-\underline r_j(x,a)
\le
341\sqrt{\frac\ell{n_j}}
\norm[I_\star^{-1}]{q_r(x,a)}.
\label{eq:reward-width}
\end{equation}
For $z=\phi^G(x,a)$, the last vector is $q_G(x,a)$.  The length of the effect interval is the sum of the factual and counterfactual interval lengths, so
\begin{equation}
\overline\Delta_j(x,a)-\underline\Delta_j(x,a)
\le
341\sqrt{\frac\ell{n_j}}
\left(
\norm[I_\star^{-1}]{q_r(x,a)}
+\norm[I_\star^{-1}]{q_G(x,a)}
\right).
\label{eq:effect-width}
\end{equation}

Using $\norm[I_\star^{-1}]{q_r(x,a)}\le\sqrt{V_\star}$ and $\norm[I_\star^{-1}]{q_r(x,a)}+\norm[I_\star^{-1}]{q_G(x,a)}\le3\sqrt{V_\star}$, the right-hand side of \eqref{eq:reward-width} is at most $341\sqrt{V_\star\ell/n_j}$, and that of \eqref{eq:effect-width} is at most
\[
3\cdot341\sqrt{\frac{V_\star\ell}{n_j}}
=1023\sqrt{\frac{V_\star\ell}{n_j}}
<1024\sqrt{\frac{V_\star\ell}{n_j}}.
\]
Taking the maximum over actions proves \eqref{eq:master-width}.
\end{proof}

\section{Technical Lemmas}

\begin{lemma}
\label{lem:logistic-geometry}
For all $u,v\in\R$,
\begin{equation}
e^{-|u-v|}\dot\mu(v)
\le
\dot\mu(u)
\le
e^{|u-v|}\dot\mu(v).
\label{eq:logistic-comparison}
\end{equation}
\end{lemma}

\begin{proof}
Differentiation gives
\[
\frac{d}{du}\log\dot\mu(u)=1-2\mu(u),
\]
whose absolute value is at most one.  Thus $\log\dot\mu$ is $1$-Lipschitz. Integrating between $u$ and $v$ and exponentiating proves \eqref{eq:logistic-comparison}.
\end{proof}

\begin{lemma}
\label{lem:one-step-kl}
For all $z_0,z_1\in\R$,
\begin{equation*}
  \operatorname{KL}\!\left(
    \operatorname{Bernoulli}(\mu(z_0))
    \,\middle\|\,
    \operatorname{Bernoulli}(\mu(z_1))
  \right)
  \le
  \frac12\sup_{z\in[z_0,z_1]}\dot\mu(z)(z_1-z_0)^2,
\end{equation*}
where $[z_0,z_1]$ denotes the segment between $z_0$ and $z_1$. Note that $z_1 < z_0$ is also possible, but we abuse the interval notation here.
\end{lemma}

\begin{proof}
Write $P_z=\operatorname{Bernoulli}(\mu(z))$ and set $h=z_1-z_0$. For $A(z) = \log(1+e^z)$, we have $A'(z) = \mu(z)$. Thus
\begin{align*}
  \operatorname{KL}(P_{z_0}\|P_{z_1})
  &=A(z_1)-A(z_0)-A'(z_0)h\\
  &=h^2\int_0^1(1-s)A''(z_0+sh)\,ds\\
  &\le \frac12h^2\sup_{z\in[z_0,z_1]}\dot\mu(z).
\end{align*}
\end{proof}

For a randomized policy, let $\cH_t$ denote the history immediately before the round-$t$ action is randomized, including the current context $X_t$. Conditional on $\cH_t$ and $A_t$, $R_t$ is a new Bernoulli draw.

\begin{lemma}
\label{lem:adaptive-kl}
Consider environments $j,k$ with the same context law and the same policy. Let $P_j^t,P_k^t$ be the respective distributions of $\cH_t$.  Then
\begin{equation}
  \operatorname{KL}(P_j^t\|P_k^t)
  =\E_j\!\left[
    \sum_{s=1}^{t-1}
    \operatorname{KL}\!\left(
      P_j(R_s\mid\cH_s,A_s)
      \,\middle\|\,
      P_k(R_s\mid\cH_s,A_s)
    \right)
  \right].
  \label{eq:adaptive-kl}
\end{equation}
If only actions in a set $\mathcal E$ have different reward distributions and every such conditional observation has KL at most $\gamma$, then
\begin{equation}
  \operatorname{KL}(P_j^t\|P_k^t)
  \le \gamma\,\E_jN_{\mathcal E}(t-1),
  \quad \text{where }
  N_{\mathcal E}(s)=\sum_{u=1}^s\ind\{A_u\in\mathcal E\}.
  \label{eq:adaptive-kl-count}
\end{equation}
\end{lemma}

\begin{proof}
Apply the chain rule for KL divergence to the variables $X_s,A_s,R_s$. The context terms vanish because the contexts are iid with the same distribution. Conditional on the realized history, the policy uses the same action distribution in both environments, so the action terms also vanish. Only the conditional reward terms remain, giving \eqref{eq:adaptive-kl}. Bounding each nonzero term by $\gamma$ gives \eqref{eq:adaptive-kl-count}.
\end{proof}

We use the following information-theoretic inequalities.

\begin{lemma}
\label{lem:info-tools}
The following statements hold.
\begin{enumerate}
\item \emph{\citet{bretagnolle1979estimation}; see also \citet[Theorem~14.2]{lattimore2019bandit}.} For probability measures $P,Q$ and every event $A$,
\begin{equation}
  P(A^c)+Q(A)\ge\frac12e^{-\operatorname{KL}(P\|Q)}.
  \label{eq:bh}
\end{equation}
\item \emph{Data processing \citep[Exercise~14.10]{lattimore2019bandit}.} If $Y$ is a measurable function of $H$, then $\operatorname{KL}(P_Y\|Q_Y)\le\operatorname{KL}(P_H\|Q_H)$. Adjoining an independent random seed with the same law under $P$ and $Q$ does not change KL.
\item \emph{Pairwise mutual-information bound.}  If $m\ge2$, $Z$ is uniform on $[m]$, and $H\mid\{Z=j\}\sim P_j$, then
\begin{equation}
  I(Z;H)\le\frac1{m^2}\sum_{j,k=1}^m
  \operatorname{KL}(P_j\|P_k).
  \label{eq:pairwise-mi}
\end{equation}
\item \emph{Fano \citep[Theorem~2.10.1]{cover1991elements}.} Under the setup of item 3, for every estimator $\widehat Z(H)$ taking values in $[m]$,
\begin{equation}
  \Pp(\widehat Z\ne Z)
  \ge 1-\frac{I(Z;H)+\log2}{\log m}.
  \label{eq:fano}
\end{equation}
\end{enumerate}
\end{lemma}

For item 3, with
\[
  \overline P=\frac1m\sum_{k=1}^m P_k,
\]
we have
\begin{align*}
  I(Z;H)
  &=\frac1m\sum_{j=1}^m
    \operatorname{KL}(P_j\|\overline P)\\
  &\le
    \frac1m\sum_{j=1}^m
    \frac1m\sum_{k=1}^m
    \operatorname{KL}(P_j\|P_k),
\end{align*}
where the inequality follows from convexity of $Q\mapsto\operatorname{KL}(P_j\|Q)$.  This proves \eqref{eq:pairwise-mi}. For the second claim in item 2, if $U\sim\rho$ is independent of $H$ under both $P$ and $Q$, then
\[
  \operatorname{KL}(P_H\otimes\rho\|Q_H\otimes\rho)
  =
  \operatorname{KL}(P_H\|Q_H)
  +
  \operatorname{KL}(\rho\|\rho)
  =
  \operatorname{KL}(P_H\|Q_H).
\]
Consequently, if a randomized policy is represented as $A_t=f_t(\cH_t,U_t)$ using an environment-independent seed $U_t$, data processing gives
\[
\operatorname{KL}\!\left(P_j(A_t\in\cdot)\,\middle\|\,P_k(A_t\in\cdot)
  \right)
  \le
  \operatorname{KL}(P_j^t\|P_k^t).
\]

Finally, the form \eqref{eq:fano} follows from Fano's inequality.  Writing $p_e=\Pp(\widehat Z\ne Z)$ and using $\mathsf{H}(Z)=\log m$,
\[
  \log m-I(Z;H)
  =
  \mathsf{H}(Z\mid H)
  \le
  h_2(p_e)+p_e\log(m-1)
  \le
  \log2+p_e\log m.
\]
Rearranging gives \eqref{eq:fano}.

\begin{lemma}
\label{lem:variational}
For every positive-definite matrix $M$ and vector $q$,
\begin{equation}
  q^\top M^{-1}q
  =\sup_{u\in\R^d}\left(2q^\top u-u^\top Mu\right).
  \label{eq:variational}
\end{equation}
\end{lemma}

\begin{proof}
Complete the square:
\[
  2q^\top u-u^\top Mu
  =q^\top M^{-1}q-\norm[M]{u-M^{-1}q}^2.
\]
The supremum is attained at $u=M^{-1}q$.
\end{proof}

\begin{lemma}
\label{lem:packing}
For every integer $n\ge45$ and every $m\le\lfloor e^{n/32}\rfloor$, there exist vectors
\[
  u_1,\ldots,u_m\in
  \left\{-\frac{1}{\sqrt n}, \frac{1}{\sqrt n}\right\}^n
\]
such that $u_i^\top u_j\le1/2$ whenever $i\ne j$.
\end{lemma}

\begin{proof}
Draw the vectors independently and uniformly. For a fixed pair $i\ne j$,
\[
  u_i^\top u_j=\frac1n\sum_{k=1}^n\xi_k,
\]
where the $\xi_k$ are independent Rademacher variables, taking values $+1$ and $-1$ with equal probability. Hoeffding's inequality gives
\[
  \Pp(u_i^\top u_j>1/2)
  =\Pp\!\left(\sum_{k=1}^n\xi_k>\frac n2\right)
  \le e^{-n/8}.
\]
A union bound over fewer than $m^2/2$ unordered pairs gives a failure probability of at most
\[
  \frac{m^2}{2}e^{-n/8}
  \le\frac12e^{-n/16}<1.
\]
Hence some choice of vectors satisfies all the required pairwise inequalities.
\end{proof}

\subsection*{Use of Generative-AI Tools}
The authors defined the parameter $V_\star$ that quantifies instance difficulty (Definition~\ref{def:vstar}), developed Algorithm~\ref{alg:fixed-ete}, and established its matching upper bound. GPT-5.6 Sol assisted in developing the hard instances, the lower-bound proofs, Algorithm~\ref{alg:adaptive-ete}, and the proof of its upper bound. It also assisted with manuscript polishing, reference exploration, and simulation code development. The authors take responsibility for the final arguments and presentation.

\end{document}